\documentclass[conference]{IEEEtran}
\IEEEoverridecommandlockouts
\usepackage{cite}
\usepackage{amsmath,amssymb,amsfonts}
\usepackage{amsthm}
\usepackage{algorithm}
\usepackage{algorithmic}
\usepackage{graphicx}
\usepackage{textcomp}
\usepackage{xcolor}
\usepackage{blindtext}
\usepackage{float}
\usepackage[hidelinks]{hyperref}
\usepackage{multirow}
\usepackage{subcaption}
\usepackage{tabularx}
\usepackage{booktabs}    
\usepackage{multirow}    
\usepackage{array}       
\newcolumntype{P}[1]{>{\centering\arraybackslash}p{#1}}

\newtheorem{thm}{Theorem}[section]

\newtheorem{ass}{Assumption}[section]

\newcommand*{\mycorref}[2][Corollary~]{%
  \hyperref[{#2}]{#1\ref*{#2}}%
}

\newcommand*{\mythm}[2][Theorem~]{%
  \hyperref[{#2}]{#1\ref*{#2}}%
}

\newcommand*{\myassumption}[2][Assumption~]{%
  \hyperref[{#2}]{#1\ref*{#2}}%
}

\newcommand*{\myalgref}[2][Algorithm~]{%
  \hyperref[{#2}]{#1\ref*{#2}}%
}

\newcommand*{\myeqref}[2][Equation~]{%
  \hyperref[{#2}]{#1(\ref*{#2})}%
}
\def\equationautorefname#1#2\null{%
  Equation#1(#2\null)%
}

\begin{document}

\title{Prior-Informed Neural Network Initialization: A Spectral Approach for Function Parameterizing Architectures}

%Prior-Informed Neural Network Initialization for Function Parameterizing Architectures - Extracting Component Parameters from Time Series Data for Bag-of-Function Encoder Initialization

\author{\IEEEauthorblockN{1\textsuperscript{st} David Orlando Salazar Torres, 2\textsuperscript{nd} Diyar Altinses, 3\textsuperscript{rd} Andreas Schwung}
\IEEEauthorblockA{\textit{Department of Automation Technology and Learning Systems} \\
\textit{South Westphalia University of Applied Sciences}\\
Soest, Germany \\
salazartorres.davidorlando, altinses.diyar, schwung.andreas@fh-swf.de}
}

\maketitle

\begin{abstract}
Neural network architectures designed for function parameterization, such as the Bag-of-Functions (BoF) framework, bridge the gap between the expressivity of deep learning and the interpretability of classical signal processing. However, these models are inherently sensitive to parameter initialization, as traditional data-agnostic schemes fail to capture the structural properties of the target signals, often leading to suboptimal convergence. In this work, we propose a prior-informed design strategy that leverages the intrinsic spectral and temporal structure of the data to guide both network initialization and architectural configuration. A principled methodology is introduced that uses the Fast Fourier Transform to extract dominant seasonal priors, informing model depth and initial states, and a residual-based regression approach to parameterize trend components. Crucially, this structural alignment enables a substantial reduction in encoder dimensionality without compromising reconstruction fidelity. A supporting theoretical analysis provides guidance on trend estimation under finite-sample regimes. Extensive experiments on synthetic and real-world benchmarks demonstrate that embedding data-driven priors significantly accelerates convergence, reduces performance variability across trials, and improves computational efficiency. Overall, the proposed framework enables more compact and interpretable architectures while outperforming standard initialization baselines, without altering the core training procedure.
\end{abstract}


\section{Introduction}

Neural networks have achieved remarkable success across a wide range of domains, including computer vision \cite{rombach2022high}, natural language processing \cite{raiaan2024review}, and signal modeling \cite{alqudah2025review, torres2025toward}. Their ability to approximate complex nonlinear relationships has made them a cornerstone of modern machine learning. However, despite their expressive power, neural networks remain inherently sensitive to stochastic factors during training, particularly to the initialization of their parameters \cite{pacelli2023statistical, neumayer2024effect}. The choice of initial weights plays a critical role in shaping the optimization landscape, directly influencing convergence speed, stability, and final performance.

Weight initialization has therefore long been recognized as a key component of effective neural network training. Poor initialization can lead to vanishing or exploding gradients, slow convergence, or unstable optimization dynamics \cite{lee2024improved}. Classical initialization schemes such as Xavier or Kaiming He were developed under assumptions tailored to conventional deep architectures and specific activation functions, aiming to preserve activation variance across layers. These methods have proven highly effective for a broad class of tasks, including image classification and language modeling \cite{wong2024analysis}.

However, modern neural architectures designed for function parameterization, such as those employed in the BoF framework, exhibit structural properties that deviate significantly from these assumptions. In such models, the network outputs directly parameterize families of mathematical functions used to reconstruct signals, rather than abstract latent features. As a result, traditional data-agnostic initialization strategies fail to account for the statistical and intrinsic structure of the underlying signals, often leading to slow convergence, unstable training behavior, and large performance variability across runs \cite{torres2025toward}.

While several works have explored alternative initialization strategies based on expert knowledge or domain-specific heuristics to mitigate these issues, these approaches often lack generality. They typically require substantial manual tuning to bridge the gap between the model and the data, limiting their scalability across diverse domains \cite{narkhede2022review}.

In this paper, we propose a data-driven and prior-informed design strategy for neural networks within the Bag-of-Functions framework. Rather than treating initialization as an isolated optimization heuristic, our approach leverages intrinsic spectral and temporal properties of the data to guide both network initialization and architectural configuration. Specifically, we employ the Fast Fourier Transform (FFT) to extract dominant seasonal components, which define the number of residual stages and initialize the respective encoder networks, alongside a residual-based regression strategy to parameterize trend components. By explicitly encoding these priors into both the architecture and the initial parameterization, the proposed method effectively steers the optimization process from the start of training.

This principled structural alignment not only stabilizes training but also enables a substantial reduction in encoder dimensionality without sacrificing reconstruction accuracy. As demonstrated through theoretical analysis and extensive empirical evaluation, the proposed framework leads to faster convergence, reduced parameter drift, improved computational efficiency, and more consistent performance across trials.

The main contributions of this paper are summarized as follows:
\begin{itemize}
    \item We introduce a data-driven design framework for Bag-of-Functions architectures, in which intrinsic spectral and temporal priors explicitly guide both weight initialization and architectural configuration.
    \item We show that aligning model capacity with intrinsic data complexity leads to a more efficient allocation of representational resources, eliminating unnecessary depth and enabling significantly more compact encoders without sacrificing reconstruction fidelity.
    \item We provide analytical support for the proposed framework, establishing reliability guarantees for spectral decomposition and residual-based trend estimation under finite-sample regimes.
    \item Through extensive experiments on synthetic and real-world datasets, we demonstrate that the proposed approach accelerates convergence, reduces parameter drift and performance variability across trials, and consistently outperforms state-of-the-art approaches.
\end{itemize}

The remainder of this paper is organized as follows. Section \ref{sec:related_work} reviews prior work on function parameterization and data-driven initialization. Section \ref{sec:Background_Problem Formulation} establishes the background on the BoF architecture and formally defines the initialization problem. Section \ref{sec:Informed_Initialization_Framework} introduces the general informed initialization framework, while Section \ref{sec:Estimating_Seasonal_Parameters_via_Fourier_Spectrum} details the spectral estimation strategy using the Fourier spectrum. Section \ref{sec:Sample_Size_Requirements_for_Linear_Regression} provides a theoretical analysis of sample size requirements and finite-sample guarantees for regression in noisy time series. Section \ref{sec:Evaluation} describes the experimental setup, dataset details, and presents the comprehensive evaluation results against state-of-the-art benchmarks. Finally, Section \ref{sec:Conclusion} concludes the paper and discusses future research directions.

\section{Related Work}
\label{sec:related_work}

In this section, we review prior work relevant to our proposed approach. We first discuss advancements in function parameterization, with a particular focus on neural network-based models such as the Bag-of-Functions framework. We then review existing techniques for informed initialization, emphasizing methods that incorporate prior knowledge to improve neural network optimization.

\subsection{Function Parameterization}

Function parameterization constitutes a foundational paradigm for representing and modeling complex signals and temporal processes, enabling tasks such as forecasting, reconstruction, and anomaly detection \cite{backhus2025time}. Classical approaches, including polynomial bases and Fourier series, rely on fixed functional forms, which often lack the flexibility required to capture complex, nonstationary patterns in real-world data. Neural networks, by contrast, have emerged as powerful tools for implicit function parameterization, offering significantly enhanced expressive capacity \cite{kong2025deep}.

Prominent neural architectures for time series modeling include N-BEATS, which employs a deep residual structure to perform basis expansion. In its interpretable configuration, N-BEATS explicitly decomposes signals into trend and seasonal components \cite{oreshkin2019n}. Related approaches, such as Implicit Neural Representations (INRs), learn continuous functional mappings directly from data \cite{jayasundara2025mire}. Despite their expressiveness, these models are typically initialized using standard, architecture-dependent schemes rather than priors derived from the intrinsic properties of the signals.

A particularly relevant line of work is the BoF framework, which has been extensively explored for interpretable signal modeling. Early contributions employed BoF representations to generate synthetic time series for pre-training autoencoders in data-scarce regimes \cite{klopries2022synthetic}. Subsequent extensions introduced probabilistic generative formulations, including ITF-VAE and ITF-GAN, enabling the synthesis and manipulation of high-fidelity time series via latent functional representations \cite{klopries2025itf, klopries2024itfgan}. Beyond generative modeling, BoF architectures have been adapted for resampling multi-resolution signals, handling variable-length sequences, and robust unsupervised data imputation \cite{salazar2025resampling, torres2025data}.

Recent work has further investigated architectural refinements within the BoF paradigm, including sparse parameter representations, flexible activation functions, and encoder design optimizations aimed at improving efficiency and expressivity \cite{klopries2023flexible, torres2025toward}. Collectively, these studies demonstrate the versatility and interpretability advantages of BoF-based models across a wide range of signal processing applications.

Despite these advances, most function-parameterizing architectures remain largely decoupled from the empirical structure of the signals they aim to model. While architectural refinements and inductive biases, such as periodic activations in implicit neural representations \cite{sitzmann2020implicit} or Fourier feature mappings to mitigate spectral bias \cite{tancik2020fourier}, have significantly improved expressivity, these approaches encode generic assumptions about signal structure rather than priors extracted from individual data instances. As a result, the incorporation of instance-specific structural information into the overall model design remains just as limited as the data-based complexity control on the number of function to be parameterized. This limitation highlights an unexploited opportunity to leverage classical signal analysis techniques, specifically spectral decomposition and trend estimation, not merely as preprocessing steps but as mechanisms to explicitly inform the initialization and configuration of function-parameterizing networks.

\subsection{Informed Data-driven Initialization}

Beyond purely random or architecture-dependent initialization schemes, a growing body of work has investigated unsupervised, data-driven strategies for neural network initialization. The central idea is to extract informative priors from the training data in order to provide a more favorable starting point for optimization, thereby accelerating convergence and improving final solution quality \cite{narkhede2025data}.

Several approaches leverage statistical or geometric properties of the data. For example, some methods employ clustering techniques such as k-means over candidate weight initializations, selecting centroids associated with lower forward-pass error as initialization points \cite{boongasame2025laor}. Other techniques use Principal Component Analysis (PCA) to align initial weight vectors with directions of maximal data variance, ensuring early sensitivity to dominant signal components \cite{phan2025principal}. While effective in general settings, these approaches are not tailored to architectures whose outputs directly parameterize interpretable functions.

In the context of Bag-of-Functions models, recent studies have explicitly examined the impact of initialization and demonstrated that aligning parameter distributions with valid functional ranges improves training stability and convergence behavior \cite{torres2025toward}. However, such strategies typically rely on fixed or manually designed assumptions about parameter statistics, rather than on priors extracted directly from the observed signal instances.

More broadly, existing data-driven initialization methods primarily operate at the level of statistical distributions or covariance structure, largely overlooking intrinsic signal characteristics such as dominant spectral components or temporal trends. Moreover, initialization is typically treated independently of architectural design choices, leaving open the question of how data-derived priors might jointly inform both the optimization starting point and the allocation of model capacity. Addressing this limitation, by coupling initialization with architecture sizing based on empirically observed spectral and temporal priors, constitutes an open research challenge that the present work seeks to address.

\section{Background and Problem Formulation}
\label{sec:Background_Problem Formulation}

The Bag-of-Functions framework is an encoder--decoder architecture designed for the interpretable decomposition of time series into fundamental components such as seasonality, trend, and events \cite{torres2025toward}. Unlike classic models that learn implicit representations, the BoF approach employs a neural network encoder, $E_\theta$, to explicitly map an input signal $\mathbf{x}_i = [x_i(t_1), \dots, x_i(t_T)]^\top \in \mathbb{R}^T$ to a latent vector of parameters $\mathbf{z}_i = E_\theta(\mathbf{x}_i)$. This parameter vector defines a set of analytic basis functions $\{\phi_a(\mathbf{t}; \mathbf{z}_{i,a})\}_{a=1}^{A}$, with $\mathbf{t} = [t_1,\dots,t_T]^\top$, whose linear combination reconstructs the signal through a functional decoder. In this way, the BoF formulation establishes a bridge between the expressive capacity of neural networks and the transparency of classical signal models, enabling the decomposition of complex signals into semantically meaningful components.  

Building on this formulation, the BoF model employs a residual hierarchy that sequentially isolates these distinct signal components. Each block operates on the residual of the previous block, promoting functional disentanglement and ensuring that each component is represented by its own parameterized basis functions. This residual hierarchy can be extended by stacking multiple BoF stages, each operating on the residual output of the previous one. Such a configuration, referred to as a stacked Bag-of-Functions architecture, allows for progressive refinement of the decomposition, enabling the model to capture higher-order interactions or subtle variations across components while preserving interpretability.

Formally, the decomposition process is defined by three encoder networks, namely $E_{\theta_s}$, $E_{\theta_t}$, and $E_{\theta_e}$, which infer the parameters of the seasonal, trend, and event components, respectively. Each encoder receives as input the residual of the preceding block, generating latent vectors $\mathbf{z}_{s_i}$, $\mathbf{z}_{t_i}$, and $\mathbf{z}_{e_i}$ that parameterize the analytic basis functions used for reconstruction:
\begin{align}
    \mathbf{z}_{s_i} &= E_{\theta_s}(\mathbf{x}_{i}), &
    \hat{\mathbf{x}}_{s_i} &= \sum_{j=1}^{J} \phi_j(\mathbf{t}; \mathbf{z}_{s_{i,j}}),\\
    \mathbf{z}_{t_i} &= E_{\theta_t}(\mathbf{x}_{i} - \hat{\mathbf{x}}_{s_i}), &
    \hat{\mathbf{x}}_{t_i} &= \sum_{k=1}^{K} \phi_k(\mathbf{t}; \mathbf{z}_{t_{i,k}}),\\
    \mathbf{z}_{e_i} &= E_{\theta_e}(\mathbf{x}_{i} - \hat{\mathbf{x}}_{s_i} - \hat{\mathbf{x}}_{t_i}), &
    \hat{\mathbf{x}}_{e_i} &= \sum_{l=1}^{L} \phi_l(\mathbf{t}; \mathbf{z}_{e_{i,l}}).
\end{align}

The overall signal reconstruction is then obtained as
\begin{equation}
    \hat{\mathbf{x}}_{i} = \hat{\mathbf{x}}_{s_i} + \hat{\mathbf{x}}_{t_i} + \hat{\mathbf{x}}_{e_i}.
\end{equation}

The latent representations $\mathbf{z}_{s_i}$, $\mathbf{z}_{t_i}$, and $\mathbf{z}_{e_i}$ provide compact and interpretable parameterizations of each component, capturing oscillatory, structural, and transient dynamics, respectively. Concatenating them yields a unified latent encoding of the signal,
\begin{equation}
    \mathbf{z}_i = [\mathbf{z}_{s_i}, \mathbf{z}_{t_i}, \mathbf{z}_{e_i}].
\end{equation}

The analytical families of basis functions $\phi(\mathbf{t}; \mathbf{z})$ are central to the interpretability of the BoF framework. Each family captures a specific temporal behavior associated with one component of the decomposition. The parameters $\mathbf{z}$ inferred by the encoders correspond to the coefficients or shape factors of these analytic functions, which can be expressed symbolically as $a_j$, the learned counterparts of $\mathbf{z}_{i,a}$ for each basis element.

Following the taxonomy of \cite{torres2025toward}, sinusoidal and harmonic functions model periodic patterns related to seasonality, polynomial and logarithmic forms represent slowly varying trends, and localized kernels such as Gaussian, sigmoidal, or step-like functions describe transient or event-related phenomena. Representative examples are summarized in \autoref{tab:functions}.

\begin{table}[htb]
\centering
\caption{Representative basis functions categorized by component type \cite{torres2025toward}. The symbolic parameters $a_j$ correspond to the learned parameters $\mathbf{z}_{i,a}$ inferred by the BoF encoders.}
\begin{tabular}{lll}
\hline
\textbf{Category} & \textbf{Function} & \textbf{Parameters} \\ \hline
Seasonality 
& $a_1\sin(a_2t + a_3)$ & $a_1, a_2, a_3$ \\
& $a_1\cos(a_2t + a_3)$ & $a_1, a_2, a_3$ \\
& $a_1\mathrm{sinc}(a_2t + a_3)$ & $a_1, a_2, a_3$ \\ \hline
Trend 
& $a_1$ & $a_1$ \\
& $a_1t$ & $a_1$ \\
& $a_1t^2 + a_2t$ & $a_1, a_2$ \\
& $a_1t^3 + a_2t^2 + a_3t$ & $a_1, a_2, a_3$ \\
& $a_1(1 - e^{ta_2})$ & $a_1, a_2$ \\
& $a_1(t+a_2)^{a_3}$ & $a_1, a_2, a_3$ \\
& $a_1\log(t+a_2)$ & $a_1, a_2$ \\ \hline
Events 
& $a_1\mathrm{step}(t, a_2)$ & $a_1, a_2$ \\
& $a_1 e^{-a_2(t-a_3)^2}$ & $a_1, a_2, a_3$ \\
& $a_1\tanh(a_2(t-a_3))$ & $a_1, a_2, a_3$ \\
& $a_1\mathrm{sig}(a_2(t-a_3))$ & $a_1, a_2, a_3$ \\ \hline
\end{tabular}
\label{tab:functions}
\end{table}

Given the parametric nature of the BoF representation, model performance is highly sensitive to the initialization of the encoder weights, $\theta$. Standard initialization schemes, such as Xavier or Kaiming, are inadequate for models whose outputs represent meaningful function parameters rather than abstract feature embeddings. Recent work has addressed this limitation by extending the Kaiming method with a bias term in the final encoder layer, $b^{(L)}$, drawn from a normal distribution. This modification enables the initial output parameters, $\mathbf{z}_i$ at the beginning of training, to be centered around a desired mean $\mu$ and variance $\sigma^2$ \cite{torres2025toward}:

\begin{align}
    b^{(L)} \sim \mathcal{N}(\mu, \sigma^2) \implies \mathbf{z}_i^{(0)} \approx \mathcal{N}(\mu, \sigma^2).
\end{align}

The objective is to select the hyperparameters $(\mu, \sigma^2)$ such that this initial distribution approximates the true, underlying parameter distribution of the dataset, denoted $P^*(\mathbf{z})$. However, let $\mathcal{H} = \{S, N_{\text{in}}, \mu, \sigma^2\}$ define the complete set of structural and initialization hyperparameters, where $S$ denotes the number of stacked decomposition stages and $N_{\text{in}}$ represents the input dimension (receptive field) of the trend encoders. In current practice, the configuration vector $\mathbf{h} \in \mathcal{H}$ is selected heuristically, independent of the signal properties. This lack of data specificity may lead to structural mismatch or statistical inefficiency, resulting in suboptimal convergence or poor alignment between the initial parameter space and the structure of the signals in a given dataset $\mathcal{D}$.

To address this limitation, we propose a principled, data-driven method for determining an improved configuration $\mathbf{h}^* \in \mathcal{H}$ by estimating it directly from the dataset. Specifically, we introduce a framework that establishes a mapping $\mathcal{M}: \mathcal{D} \to \mathcal{H}^*$ based on spectral and statistical analysis. We utilize the cardinality of the dominant spectral modes, $|\mathcal{K}_\tau|$, to formally derive the required number of stages $S$, while a finite-sample regression analysis determines the minimal input size $N_{\text{in}}$ required to bound the estimation variance. Concurrently, empirical estimates $(\hat{\mu}_{\text{data}}, \hat{\sigma}^2_{\text{data}})$ are computed to inform the weight initialization. This yields a model topology and state explicitly tailored to the intrinsic structure of the data, providing a more robust and efficient starting point for training BoF models.

\section{Informed Initialization Framework}
\label{sec:Informed_Initialization_Framework}

The proposed framework extends the Bag-of-Functions architecture through a data-informed strategy that governs both design and initialization. As illustrated in \autoref{fig:framework}, extracted seasonal and trend statistics serve as priors that explicitly determine the structural topology by defining the number of decomposition stages and the trend encoder input sizes, while simultaneously guiding the initialization of the encoder weights. This alignment ensures the model is instantiated with properties consistent with the empirical data structure, crucially enabling the subsequent training phase to resolve the complex nonlinear temporal behaviors and event-driven variations that cannot be represented by purely spectral or linear methods.

\begin{figure}[htb]
    \centering
    \includegraphics[width=\linewidth]{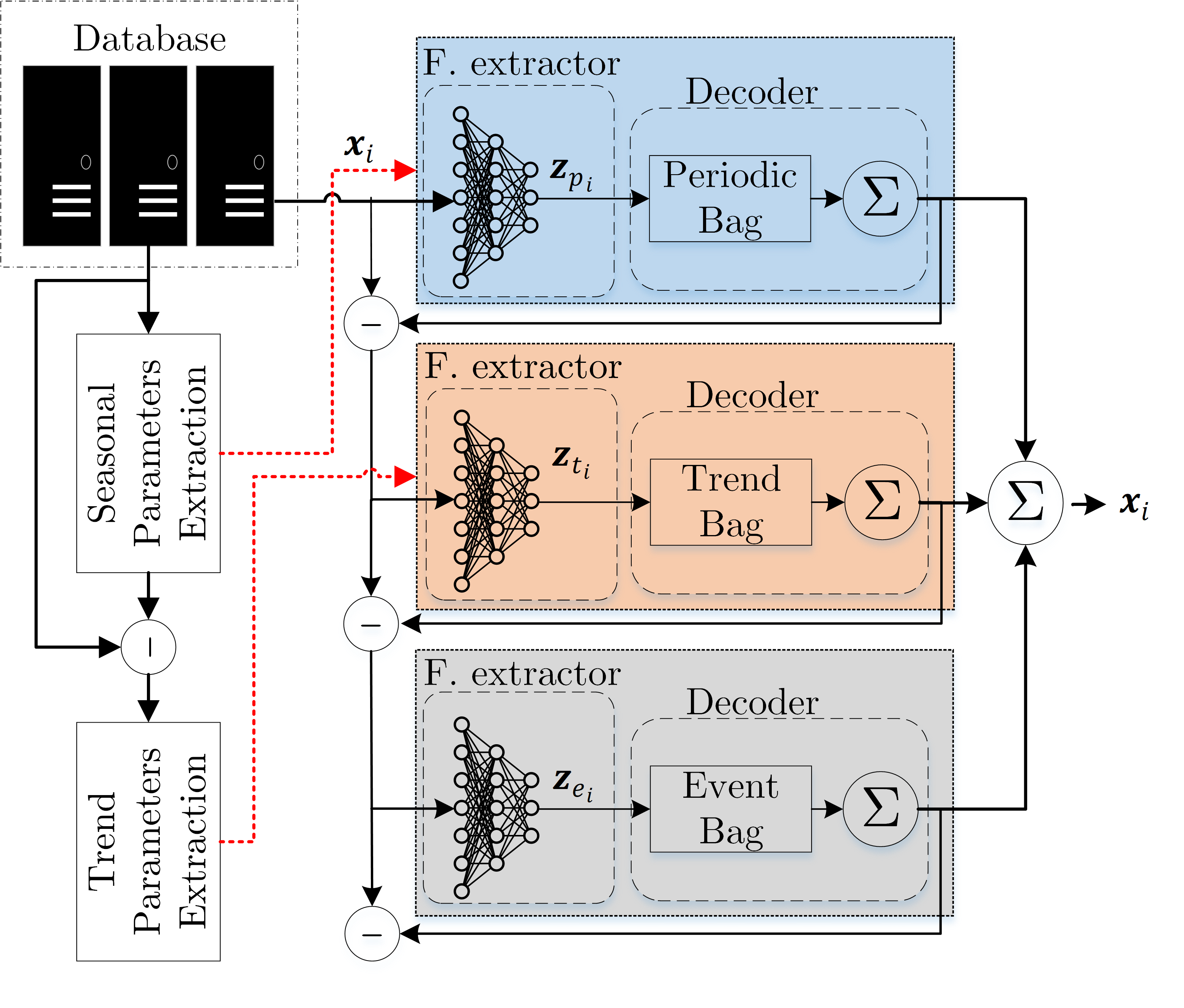}
    \caption{Overview of the proposed data-informed framework. Extracted seasonal and trend statistics serve as priors that jointly determine the architectural topology, including stage depth and input dimensions, and initialize the Bag-of-Functions encoder weights to align the model with the data structure before training.}
    \label{fig:framework}
\end{figure}

Formally, the framework operates sequentially to establish a mapping from the dataset $\mathcal{D}$ to the improved configuration $\mathcal{H}^*$. First, spectral analysis via the Fast Fourier Transform (FFT) identifies the set of dominant frequencies $\mathcal{K}(\tau)$. The cardinality of this set, $|\mathcal{K}(\tau)|$, directly informs the required number of stages $S$, while the empirical mean and variance of the corresponding spectral coefficients constitute the initialization statistics, denoted $(\hat{\mu}_{\text{data}}, \hat{\sigma}^2_{\text{data}})$. These values define the data-driven priors for initializing the seasonal Bag-of-Functions encoder:
\begin{equation}
    \mathbf{z}_{s}^{(0)} \sim \mathcal{N}(\hat{\mu}_{\text{data}}, \hat{\sigma}^2_{\text{data}}).
\end{equation}

Subsequently, the trend information is estimated through linear regression applied to the residual signal obtained after removing the identified seasonal components. This analysis yields two critical outcomes regarding both the architecture and the optimization state. On the structural front, the estimated residual noise variance $\hat{\sigma}_\varepsilon^2$ and the target slope tolerance $\delta$ are used to derive the minimum input dimension $N_{\text{in}}$ required for reliable trend estimation, utilizing the finite-sample guarantees derived in Section \ref{sec:Sample_Size_Requirements_for_Linear_Regression}. Concurrently, the empirical statistics of the slope and bias coefficients update the priors $(\hat{\mu}_{\text{data}}, \hat{\sigma}^2_{\text{data}})$ specifically for the trend encoder. This data-driven alignment ensures that the encoder outputs lie within the valid parameter space of the target signals from the first iteration, mitigating the risk of spectral bias or vanishing gradients associated with random initialization.

The overall procedure for extracting these hyperparameters is summarized in Algorithm \ref{alg:framework}. The algorithm processes the dataset to derive the configuration vector $\mathbf{h}^*$, which is then used to instantiate and initialize the Bag-of-Functions architecture.

\begin{algorithm}
\caption{Data-Driven Configuration and Initialization}
\label{alg:framework}
\begin{algorithmic}[1]
\REQUIRE Dataset $\mathcal{D} = \{y^{(i)}\}_{i=1}^N$ of $N$ time series
\STATE Identify dominant spectral modes $\mathcal{K}$ from dataset $\mathcal{D}$
\STATE Set number of decomposition stages $S \gets |\mathcal{K}|$
\STATE Initialize empty lists $\mathcal{P}_{\text{t}} \gets []$, $\mathcal{P}_{\text{s}} \gets []$
\FOR{$i = 1, \dots, N$}
    \STATE Estimate seasonal parameters 
        $\hat{\theta}_{\text{s}}^{(i)}$ for modes $\mathcal{K}$
    \STATE Compute seasonal component 
        $\hat{s}^{(i)}(t) \gets f_{\text{s}}(t;\hat{\theta}_{\text{s}}^{(i)})$
    \STATE Form residual 
        $\tilde{y}^{(i)}(t) \gets y^{(i)}(t) - \hat{s}^{(i)}(t)$
    \STATE Estimate trend parameters 
        $\hat{\theta}_{\text{t}}^{(i)} = (\hat{a}^{(i)}, \hat{b}^{(i)})$ on $\tilde{y}^{(i)}$
    \STATE Append $\hat{\theta}_{\text{s}}^{(i)}$ to $\mathcal{P}_{\text{s}}$ and $\hat{\theta}_{\text{t}}^{(i)}$ to $\mathcal{P}_{\text{t}}$
\ENDFOR
\STATE Compute input dimension $N_{\text{in}}$ from residual statistics
\STATE Compute empirical priors $(\hat{\mu}_{\text{data}}, \hat{\sigma}^2_{\text{data}})$ from $\mathcal{P}_{\text{s}}$ and $\mathcal{P}_{\text{t}}$
\STATE Configure and initialize Bag-of-Functions architecture
\end{algorithmic}
\end{algorithm}

The practical realization of this framework relies on translating these statistical priors into concrete architectural and parametric specifications. In the following sections, we treat the Fourier spectrum not merely as a feature extractor, but as a topological blueprint: the distribution of spectral energy reveals the signal's intrinsic multi-resolution structure, effectively prescribing the hierarchical depth ($S$) based on signal complexity rather than heuristics. Crucially, this analytical foundation acts as a complementary scaffold rather than a substitute; it allows the Bag-of-Functions model to bypass the learning of trivial linear components and focus its expressive power on resolving the complex non-stationary dynamics and transient events that remain inaccessible to classical methods. Combined with the bias-based initialization adapted from \cite{torres2025toward}, this approach ensures a stable starting state. Accordingly, Section \ref{sec:Estimating_Seasonal_Parameters_via_Fourier_Spectrum} details this spectral mapping strategy, while Section \ref{sec:Sample_Size_Requirements_for_Linear_Regression} establishes the theoretical bounds for stable trend estimation.

\section{Estimating Seasonal Parameters via Fourier Spectrum}
\label{sec:Estimating_Seasonal_Parameters_via_Fourier_Spectrum}

This section introduces a principled framework for estimating seasonal priors from the frequency domain. The analysis identifies dominant periodic components whose empirical statistics are used to initialize the seasonal encoder of the Bag-of-Functions model. Beyond parameter estimation, the spectral distribution of energy reveals the multi-resolution structure of the data, guiding the design of hierarchical seasonal stages. Unlike conventional Fourier methods, which focus on linear reconstruction, this approach leverages spectral descriptors to inform the initialization and configuration of the BoF architecture, enabling it to exploit its ability to model nonlinear and event-driven temporal dynamics.

\subsection{Spectral Modeling and Frequency Selection}

To derive empirical priors for the seasonal encoder, each time series in the dataset is modeled as the superposition of three components:
\begin{equation}\label{eq:input}
    y_t = s_t + t_t + \varepsilon_t, \qquad t = 1, \dots, n,
\end{equation}
where $s_t$ represents the seasonal component associated with periodic fluctuations, $t_t$ denotes the slowly varying trend, and $\varepsilon_t$ is additive noise. The objective of this stage is to extract information about $s_t$ from the frequency domain, providing a compact summary of the dominant periodic behavior in the data.

The spectral content of $y_t$ is analyzed using the Discrete Fourier Transform (DFT),
\begin{equation}
    Y(\omega_k) = \sum_{t=1}^{n} y_t \, e^{-2\pi i \omega_k t}, 
    \qquad \omega_k = \frac{k}{n}, \quad k = 0, 1, \dots, n-1,
\end{equation}
whose squared magnitude defines the periodogram,
\begin{equation}
    I_y(\omega_k) = \frac{1}{n} |Y(\omega_k)|^2.
\end{equation}
The periodogram quantifies the distribution of signal energy across frequencies, enabling the identification of peaks associated with dominant oscillatory patterns. These frequencies capture the primary components of $s_t$, while the trend $t_t$ and noise $\varepsilon_t$ contribute mainly to the low- and high-frequency regions, respectively.

To isolate meaningful periodicities, we employ a relative-power thresholding criterion. Let
\begin{equation}
    P_{\max} = \max_{0 \le k \le \lfloor n/2 \rfloor} I_y(\omega_k)
\end{equation}
be the maximum spectral power, and define a threshold parameter $\tau \in (0,1)$, typically $\tau = 0.2$. The set of dominant frequencies is then
\begin{equation}
    \mathcal{K}(\tau) = \left\{ k : I_y(\omega_k) \ge \tau P_{\max}, \quad 0 \le k \le \lfloor n/2 \rfloor \right\}.
\end{equation}

For each retained frequency $\omega_k \in \mathcal{K}(\tau)$, the amplitude $A_k = 2|C_k|$ and phase $\phi_k = \arg(C_k)$ are computed from $C_k = \frac{1}{n}Y(\omega_k)$. These quantities represent coarse estimates of the seasonal characteristics of the dataset. The empirical distributions of $\{A_k, \omega_k, \phi_k\}$ across all time series define the statistical priors $(\hat{\mu}_{\text{data}}, \hat{\sigma}^2_{\text{data}})$, which are used to initialize the bias parameters of the BoF seasonal encoder. Moreover, the global spectral distribution provides an estimate of how many distinct frequency bands are relevant, offering a data-driven guideline for configuring the number of seasonal stages in stacked BoF architectures.

\subsection{Spectral Energy Ratio as a Diagnostic Metric}

To quantify the informativeness of the identified frequencies, we compute the proportion of signal energy explained by the retained components:
\begin{equation}
    \rho_{\text{spec}} = \frac{\sum_{k \in \mathcal{K}(\tau)} |C_k|^2}{\sum_{k=0}^{n-1} |C_k|^2}.
\end{equation}
A high value of $\rho_{\text{spec}}$ indicates that the time series exhibits strong periodicity, suggesting that the spectral priors are reliable indicators for encoder initialization. Conversely, a low $\rho_{\text{spec}}$ implies weak seasonality, in which case the model relies more heavily on trend and event encoders during training. This measure thus serves as a practical diagnostic tool for evaluating the spectral richness of the dataset and for adapting the initialization strength of the seasonal stage. The deseasonalized residual obtained after this analysis serves as the input for estimating trend parameters, as described in the following section.

\section{Sample Size Requirements for Linear Regression in Noisy Time Series}
\label{sec:Sample_Size_Requirements_for_Linear_Regression}

In order to capture the underlying trend across the dataset, we apply a linear regression model to each time series sample. This procedure provides a simple and interpretable way to characterize the local direction and strength of change, which can then be aggregated or used for the informed initialization of the Bag-of-Functions architecture. While many time series exhibit additional seasonal or cyclical patterns, the linear regression applied at the segment level serves as a first-order approximation of the local trend, effectively separating long-term drift from random noise.

Since the number of regressions to be performed can be extremely large, it becomes crucial to understand how many observations within each segment are actually required to obtain reliable parameter estimates. In particular, we are interested in quantifying the minimum segment length needed to estimate the slope and bias with a prescribed level of accuracy and confidence, thereby providing a principled trade-off between statistical precision and computational or data-collection cost. Moreover, this analysis also offers practical guidance for model design: the minimum segment length directly translates into the temporal window size required by the trend encoder, defining the effective receptive field of the neural network necessary to capture stable linear dynamics in noisy environments.

\subsection{Model Setup}

We consider a time series segment of length $n$ with observations
\begin{align}
    y_i = a + b\, t_i + \varepsilon_i, \qquad i = 1,2,\ldots,n,
\end{align}

where $a$ is the bias, $b$ is the slope, $t_i$ denotes the time index (assumed equally spaced), and $\varepsilon_i$ are stochastic noise terms with zero mean. Note that we subtract the seasonal components from \myeqref{eq:input} as previously defined.

The ordinary least squares (OLS) estimator $\hat\beta = (\hat a, \hat b)^\top$ minimizes the squared residuals and admits the closed form
\begin{align}
    \hat\beta = (X^\top X)^{-1} X^\top y,
\end{align}
where $X$ is the design matrix with rows $(1, t_i)$.

A central quantity in this analysis is the centered sum of squares of the time indices,
\begin{align}
    S_{xx} = \sum_{i=1}^n (t_i - \bar t)^2, \qquad 
    \bar t = \frac{1}{n} \sum_{i=1}^n t_i.
\end{align}
This term quantifies how widely the design points are distributed along the time axis.  
A larger spread (higher $S_{xx}$) leads to more precise estimates of the slope $b$, since the regression has greater leverage to separate signal from noise.

When the time indices are equally spaced, i.e., $t_i = i$ for $i=1,\ldots,n$, the mean time index is
\begin{align}
    \bar t = \frac{1}{n}\sum_{i=1}^n i = \frac{n+1}{2}.
\end{align}
Substituting this into the definition of $S_{xx}$ gives
\begin{align}
    S_{xx} &= \sum_{i=1}^n \left(i - \frac{n+1}{2}\right)^2 
    = \sum_{i=1}^n i^2 - n\left(\frac{n+1}{2}\right)^2. \label{eq:sxx_expand}
\end{align}
Using the standard identity for the sum of squares,
\begin{align}
    \sum_{i=1}^n i^2 = \frac{n(n+1)(2n+1)}{6},
\end{align}
and substituting into \myeqref{eq:sxx_expand} yields
\begin{align}
    S_{xx} 
    &= \frac{n(n+1)(2n+1)}{6} - n\left(\frac{n+1}{2}\right)^2 
    = \frac{n(n^2 - 1)}{12}.
\end{align}

For large $n$, this expression scales as $S_{xx} \sim \frac{n^3}{12}$, indicating that the effective spread of the design and, consequently, the statistical leverage of the regression, increase cubically with the number of observations. This property directly determines the rate at which the slope estimator converges in the subsequent analysis. The closed-form expression for $S_{xx}$ derived here will serve as the foundation for quantifying the variance and concentration properties of the slope estimator in the next subsection.

\subsection{Estimation Error of the Slope}

From the general OLS solution $\hat\beta = (X^\top X)^{-1}X^\top y$ and the linear model $y = X\beta + \varepsilon$, we can express the estimation error as
\begin{align}
    \hat\beta - \beta = (X^\top X)^{-1}X^\top \varepsilon.
\end{align}
For the slope component, corresponding to the second entry of $\hat\beta$, this gives
\begin{align}
    \hat b - b = e_2^\top (X^\top X)^{-1}X^\top \varepsilon,
\end{align}
where $e_2 = (0,1)^\top$ selects the slope coordinate.  
Using the structure of the design matrix
\[
X =
\begin{bmatrix}
1 & t_1 \\
\vdots & \vdots \\
1 & t_n
\end{bmatrix},
\]
and applying the standard OLS algebra, one obtains the explicit scalar form
\begin{align}
    \hat b = \frac{\sum_{i=1}^n (t_i-\bar t)y_i}{\sum_{i=1}^n (t_i-\bar t)^2}.
\end{align}
Substituting $y_i = a + b t_i + \varepsilon_i$ into this expression yields
\begin{align}
    \hat b - b
    &= \frac{\sum_{i=1}^n (t_i-\bar t)(a + b t_i + \varepsilon_i)}{S_{xx}} - b \\
    &= \frac{a\sum_{i=1}^n (t_i-\bar t) + b\sum_{i=1}^n (t_i-\bar t)^2 + \sum_{i=1}^n (t_i-\bar t)\varepsilon_i}{S_{xx}} - b.
\end{align}
The first term vanishes because $\sum_i (t_i-\bar t)=0$, and the second term cancels with $b$ after dividing by $S_{xx}$.  
Thus we obtain the key relation
\begin{align}
    \hat b - b = \frac{1}{S_{xx}}\sum_{i=1}^n (t_i - \bar t)\,\varepsilon_i. \label{eq:slope_error}
\end{align}

Equation \eqref{eq:slope_error} shows that the slope estimation error is a weighted average of the noise terms, where the weights depend linearly on the centered time indices. Observations farther from the mean time point contribute more heavily to the estimate, which intuitively explains why a wider time span yields a more stable slope estimate.

Assuming that the noise terms are independent with $\mathbb{E}[\varepsilon_i]=0$ and $\mathrm{Var}(\varepsilon_i)=\sigma^2$, the variance of the slope estimator follows directly:
\begin{align}
    \mathrm{Var}(\hat b)
    &= \mathrm{Var}\!\left(\frac{1}{S_{xx}}\sum_{i=1}^n (t_i - \bar t)\varepsilon_i\right) \\
    &= \frac{1}{S_{xx}^2} \sum_{i=1}^n (t_i - \bar t)^2 \mathrm{Var}(\varepsilon_i) \\
    &= \frac{\sigma^2}{S_{xx}}.
\end{align}
Substituting the equidistant design result $S_{xx} = n(n^2 - 1)/12$ gives the asymptotic behavior
\begin{align}
    \mathrm{Var}(\hat b) \sim \frac{12\sigma^2}{n^3}, \qquad n \to \infty.
\end{align}
Hence, the standard deviation of $\hat b$ decreases at the rate $n^{-3/2}$—a much faster convergence than the usual $n^{-1/2}$ rate obtained when estimating a mean.  
This acceleration arises because slope estimation benefits from the cubic increase in the time-index spread, which amplifies the information gained from additional observations.

\subsection{Concentration Bounds and Finite-Sample Guarantees}

While variance calculations describe asymptotic precision, they do not quantify the reliability of estimates for finite sample sizes. To obtain non-asymptotic guarantees, we study the distribution of the slope estimator under sub-Gaussian noise assumptions.

\begin{ass}
\label{assump}
The noise variables $\varepsilon_i$ are independent, centered, and sub-Gaussian with variance proxy $\sigma^2$, i.e.,
\begin{align}
    \mathbb{E}[e^{\lambda \varepsilon_i}] \le \exp\!\left(\frac{\lambda^2\sigma^2}{2}\right), \quad \forall \lambda \in \mathbb{R}.
\end{align}
This assumption includes the Gaussian case but also accommodates other light-tailed distributions.
\end{ass}

\begin{thm}
\label{thm:thefirst}
Under \myassumption{assump}, for any tolerance level $\delta > 0$,
\begin{align}
    \mathbb{P}(|\hat b - b| > \delta) \le 2 \exp\!\left(-\frac{\delta^2 S_{xx}}{2\sigma^2}\right).
\end{align}
\end{thm}

\begin{proof}
From \myeqref{eq:slope_error}, the slope estimation error can be written as
\begin{align}
    \hat b - b = \frac{1}{S_{xx}}\sum_{i=1}^n (t_i-\bar t)\varepsilon_i
               = \sum_{i=1}^n w_i \varepsilon_i,
\end{align}
where the weights $w_i = (t_i - \bar t)/S_{xx}$ are deterministic and satisfy $\sum w_i^2 = 1/S_{xx}$.

By the standard closure property of sub-Gaussian variables, any weighted sum of independent sub-Gaussian terms remains sub-Gaussian. In particular, if each $\varepsilon_i$ is sub-Gaussian with variance proxy $\sigma^2$, then the sum
\[
Z = \sum_{i=1}^n w_i \varepsilon_i
\]
is sub-Gaussian with variance proxy
\[
\sigma_Z^2 = \sigma^2 \sum_{i=1}^n w_i^2 = \frac{\sigma^2}{S_{xx}}.
\]

By definition of sub-Gaussianity, for any $u>0$,
\begin{align}
    \mathbb{P}(|Z| > u) \le 2\exp\!\left(-\frac{u^2}{2\sigma_Z^2}\right)
    = 2\exp\!\left(-\frac{u^2 S_{xx}}{2\sigma^2}\right).
\end{align}

Finally, substituting back $Z = \hat b - b$ and setting $u = \delta$ yields the stated concentration inequality.
\end{proof}

\mythm{thm:thefirst} provides a direct link between the design spread $S_{xx}$, the noise level $\sigma^2$, and the desired accuracy $\delta$. The larger the effective spread $S_{xx}$, the tighter the estimator concentrates around the true slope.

\medskip

A practical corollary gives an explicit sample-size requirement.  
To ensure that $|\hat b - b| \le \delta$ with probability at least $1-\alpha$, it suffices that
\begin{align}
    S_{xx} \ge \frac{2\sigma^2 \log(2/\alpha)}{\delta^2}. \label{eq:sxx_requirement}
\end{align}
In the equidistant case $t_i=i$, we have $S_{xx} = n(n^2-1)/12 \approx n^3/12$, and substituting into \myeqref{eq:sxx_requirement} gives
\begin{align}
    n \ge \left(\frac{24\sigma^2 \log(2/\alpha)}{\delta^2}\right)^{1/3}. \label{eq:n_requirement}
\end{align}
Hence, the required segment length grows only as the two-thirds power of the signal-to-noise ratio $(\sigma/\delta)$. This is considerably milder than the quadratic scaling typical of mean estimation, reflecting the higher efficiency of slope estimation due to the cubic growth of $S_{xx}$.

\medskip

For illustration, consider a target tolerance $\delta = 0.1$, confidence level $95\%$ ($\alpha = 0.05$), and noise variance $\sigma^2 = 1$.
Substituting into \myeqref{eq:n_requirement} gives
\begin{align}
    n \ge \left(\frac{24\,\ln(40)}{0.1^2}\right)^{1/3} \approx 20.7.
\end{align}
Thus, a segment length of \textbf{$n = 21$} equidistant observations is sufficient to guarantee the desired estimation accuracy under independent Gaussian noise. This provides a practical lower bound for the temporal window size required by the trend encoder in the Bag-of-Functions framework.

\subsection{Bias Estimation}

We now turn to the bias estimator $\hat a$. Unlike the slope, the bias does not directly benefit from the design spread of the time indices. Instead, it primarily reflects the average noise level across the segment, which leads to a different scaling behavior. From the standard OLS relation $\hat a = \bar y - \hat b \bar t$, substituting $y_i = a + b t_i + \varepsilon_i$ gives
\begin{align}
    \hat a - a = \frac{1}{n}\sum_{i=1}^n \varepsilon_i - (\hat b - b)\bar t.
\end{align}
The first term corresponds to the sample mean of the noise, while the second term captures the propagated uncertainty from the slope estimation, scaled by the mean of the time indices. In the equidistant case, these two components are uncorrelated, and thus their variances add.

Assuming independent errors with variance $\sigma^2$, we obtain
\begin{align}
    \mathrm{Var}(\hat a) = \sigma^2\left(\frac{1}{n} + \frac{\bar t^2}{S_{xx}}\right).
\end{align}
The first term arises from the sample mean variance $\mathrm{Var}(\bar\varepsilon) = \sigma^2/n$, while the second follows from $\mathrm{Var}(\hat b - b) = \sigma^2/S_{xx}$, scaled by $\bar t^2$. Cross-terms vanish due to the orthogonality of the OLS decomposition.

For equidistant time indices $t_i=i$, we have $\bar t = (n+1)/2$ and $S_{xx} \sim n^3/12$. Substituting these expressions yields
\begin{align}
    \mathrm{Var}(\hat a) \sim \frac{\sigma^2}{n}, \qquad n \to \infty.
\end{align}
Thus, the bias estimator converges at the classical $n^{-1/2}$ rate, identical to the sample mean. Unlike the slope, it does not benefit from the cubic increase of $S_{xx}$, since it captures the baseline level rather than the rate of change.

From a practical standpoint, this distinction is relevant. To estimate the slope within tolerance $\delta$, one needs on the order of $n \asymp (\sigma/\delta)^{2/3}$ observations, whereas for the bias the requirement increases to $n \asymp (\sigma/\delta)^2$. Hence, achieving a high-precision bias estimate is significantly more demanding than estimating the slope.

\subsection{Implications for Architecture Design}
\label{subsec:Implications_for_Architecture}

While the asymptotic properties of the ordinary least squares estimator are classical results in statistical theory, their application within the context of neural network design offers an alternative to the heuristic selection of hyperparameters. Typically, the input size of a trend encoder is determined via costly grid searches or arbitrary rule-of-thumb choices. The derivation provided in Theorem \ref{thm:thefirst} establishes a direct analytical link between the physical properties of the signal (noise-to-slope ratio) and the structural requirements of the architecture.

A critical insight from this analysis is the disparity in convergence rates between the slope ($\sim n^{-3/2}$) and the bias ($\sim n^{-1/2}$). In the context of the Bag-of-Functions architecture, the primary goal of the trend encoder is to capture the dynamics of the signal, represented by the slope, rather than its absolute baseline level. The cubic scaling of the design spread $S_{xx}$ implies that the encoder can reliably learn local derivatives with relatively compact temporal windows. For instance, the derived bound suggests that a window of approximately $n \approx 20$ samples is sufficient to stabilize the gradient estimation even under significant noise ($\sigma \approx 1$), whereas relying on bias convergence would require significantly larger inputs.

This theoretical finding validates the efficiency of the proposed framework: by allowing the linear regression layer to focus on the rapidly converging slope component, we can minimize the dimensionality of the input layer $N_{\text{in}}$ without sacrificing precision. The slower-converging bias component acts merely as a bias term, which neural networks are naturally equipped to correct through their internal bias weights during the training phase via backpropagation. Consequently, \myeqref{eq:n_requirement} serves as a formal lower bound for the receptive field size, ensuring that the model is instantiated with sufficient informational capacity to distinguish trends from noise prior to any optimization.

\section{Evaluation}
\label{sec:Evaluation}

This section presents the experimental evaluation of the proposed framework. 
The analysis is structured in three parts. 
First, the datasets used for training and validation are described, including a controlled synthetic dataset and two real-world datasets representing distinct signal domains. 
Next, the experimental setup outlines the implementation details, parameter configurations, and evaluation protocol employed. 
Finally, the results section reports both quantitative and qualitative findings, highlighting the effectiveness of the prior-informed initialization and the interpretability achieved by the Bag-of-Functions architecture. 
The evaluation further demonstrates how data-derived information guides key architectural choices, such as the number of stages and the encoder window size, thereby achieving a principled balance between interpretability and efficiency.

\subsection{Dataset}

To validate the proposed data-informed strategy for architectural design and initialization, we employ both synthetic and real-world benchmarks. These datasets encompass diverse temporal dynamics, enabling a systematic evaluation of how extracted spectral and trend priors effectively guide the instantiation of the BoF encoders.

The synthetic dataset was constructed to provide a reference benchmark with known spectral and temporal parameters. Each signal \( y_i(t) \) combines periodic components, a linear trend, and a localized transient event, modeled as
\begin{equation}
\begin{split}
    y_i(t) = &\; \sum_{k=1}^{3} A_{i,k} 
    \sin\!\big( 2\pi f_{i,k} t + \phi_{i,k} \big)
    + (\alpha_{i,0} + \alpha_{i,1} t) \\[1mm]
    &\; + \beta_i 
    \exp\!\left[-\frac{(t - \mu_i)^2}{2\sigma_{e}^2}\right]
    + \varepsilon_i(t).
\end{split}
\end{equation}
The spectral parameters are drawn independently, with frequencies \( f_{i,k} \) normally distributed around centers \(\{3.5, 7.5, 12.0\}\,\text{Hz}\), amplitudes \( A_{i,k} \) following a positive normal distribution, and phases \( \phi_{i,k} \) uniform in \([0, 2\pi]\). Both trend and event parameters are sampled from bounded uniform distributions to introduce variability, while the additive noise follows \( \varepsilon_i(t) \sim \mathcal{N}(0, 0.01^2) \). Comprising \( N = 2000 \) signals sampled at \( T = 100 \) equidistant points over the normalized interval \(t \in [0, 1]\), this dataset enables a direct quantitative evaluation of the estimated priors against the known ground truth.

To complement this synthetic benchmark, we utilize two real-world datasets representing divergent temporal challenges. The first, the PJM Interconnection dataset (1998–2001), captures macro-scale grid dynamics through hourly demand records partitioned into weekly segments to isolate recurrent consumption patterns. The second comprises thermal power plant (TPP) output from a district heating network, sampled at hourly intervals between 2016 and 2020. By focusing on the winter subset of the TPP data, we isolate the prominent periodic and trend components inherent to heating demand, allowing the framework to demonstrate its ability to extract meaningful architectural priors across diverse and complex empirical domains.

\subsection{Experimental Setup}

%The experimental methodology established in this study follows a two-stage pipeline designed to validate the integration of structural signal priors into the neural architecture. The process begins with a data-driven diagnosis phase, utilizing spectral and trend analyses to characterize the underlying intrinsic properties of the raw training data. These extracted statistics serve as the direct inputs for configuring the model's architecture and initialization. Following this configuration, we proceed to the training and evaluation phase to quantify the gains in convergence speed, stability, and predictive accuracy.

To assess the impact of the extracted priors on both the architectural design decisions and the parameter initialization strategies, we benchmark four distinct configurations of the Bag-of-Functions network. The first variant, denoted as BoF, serves as the domain-agnostic baseline where the encoders are instantiated using standard default initialization schemes, representing a scenario devoid of any specific structural priors. The second variant, Heuristics BoF (H-BoF), adopts the initialization philosophy proposed in \cite{torres2025toward}; however, while this approach utilizes specific heuristics for initialization, it remains constrained to a fixed architecture that does not incorporate design decisions derived from the dataset's specific characteristics. Moving toward structurally informed configurations, the Informed BoF (I-BoF) architecture extends this initialization strategy by replacing generic heuristics with data-driven values. Crucially, this variant leverages the spectral analysis to dynamically determine the number of stages of the model and to initialize the seasonal encoders, while the trend component retains a standard configuration. Finally, the Informed Trend BoF (IT-BoF) represents the complete proposed framework where the architecture is fully adaptive: the depth is defined by the frequency analysis, and the trend encoder input size is dimensioned according to the residual trend volatility. Furthermore, this network is fully initialized with both the spectral seasonal priors and the OLS-derived trend priors, ensuring the optimization starts in close proximity to the true signal parameters. To account for the inherent stochasticity of weight initialization and the optimization landscape, we executed 10 independent training trials for each dataset-model pair using distinct random seeds. All models were trained using the Adam optimizer with a learning rate of $\eta=1\times10^{-3}$ and a batch size of $B=16$, governed by a Mean Squared Error (MSE) loss function with an early stopping mechanism to prevent overfitting.

To evaluate the computational efficiency and practical suitability of the proposed architectures, we conducted a standardized benchmarking protocol on a workstation equipped with an Intel Core i7-11700 CPU, 64 GB RAM, and an NVIDIA RTX 5000 GPU. All models were implemented in PyTorch, and computational complexity was estimated using the thop library. Inference performance was measured under controlled conditions using real input samples from the test sets to ensure representative runtime behavior. %The evaluation reports model size (number of learnable parameters), computational complexity (KFLOPs), inference latency, and throughput, providing a comprehensive assessment of efficiency and scalability.

\subsection{Data-Driven Prior Extraction and Architecture Configuration}

We apply the prior extraction framework to the synthetic and real-world datasets in order to instantiate the Bag-of-Functions architecture before training. The objective is twofold: (i) to derive empirical priors for encoder initialization, and (ii) to determine the architectural depth and trend encoder input size directly from data.

Spectral estimation is first applied to the synthetic dataset described in \autoref{sec:Evaluation}. \autoref{fig:FFT_Synthetic} reports the average periodogram and the spectral heatmap across all $N=2000$ samples. Applying the relative-power thresholding criterion ($\tau=0.2$) identifies three dominant modes at $3.48$, $7.56$, and $12.29$ Hz, closely matching the ground-truth frequencies.
These components account for a cumulative spectral energy ratio of $\rho_{\text{spec}} = 0.873$, indicating strong periodic structure. Accordingly, the BoF architecture is instantiated with $|\mathcal{K}|=3$ stages, each initialized using the empirical statistics of the corresponding spectral band.

\begin{figure}
    \centering
    \includegraphics[width=1\linewidth]{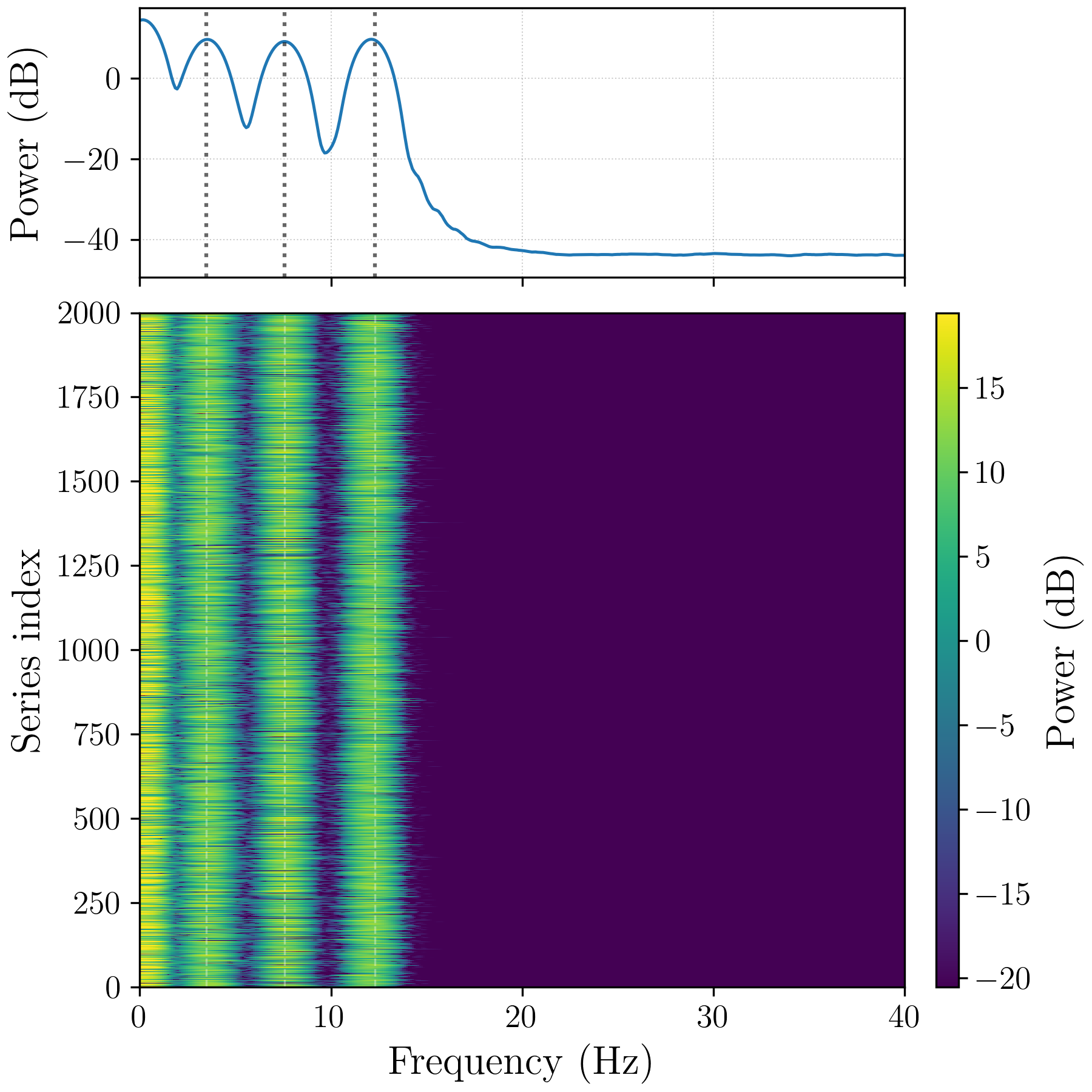}
    \caption{Spectral analysis of the synthetic dataset used for prior extraction. The average periodogram identifies dominant modes at $3.48$ Hz, $7.56$ Hz, and $12.29$ Hz, complemented by the spectral heatmap across all $N=2,000$ samples. The cumulative energy ratio reaches $\rho_{\text{spec}} = 0.873$ by combining these three identified modes.}
    \label{fig:FFT_Synthetic}
\end{figure}

Following seasonal extraction, a coarse deseasonalization is performed to estimate trend priors. As illustrated in \autoref{fig:Trend_estimation}, the removal of dominant harmonics exposes the underlying drift, enabling a reliable estimation of the residual noise variance. Using the finite-sample bound derived in \autoref{eq:n_requirement} with tolerance $\delta=0.2$, the resulting stability criterion yields an optimal trend input size of $n_{opt}=11$. These observations are selected equidistantly to maximize the design spread $S_{xx}$, yielding accurate slope estimates while reducing the input dimensionality by $89\%$ compared to a full-window approach.

\begin{figure}[htb]
     \centering
     \begin{subfigure}[b]{0.49\linewidth}
         \centering
         \includegraphics[width=1\textwidth]{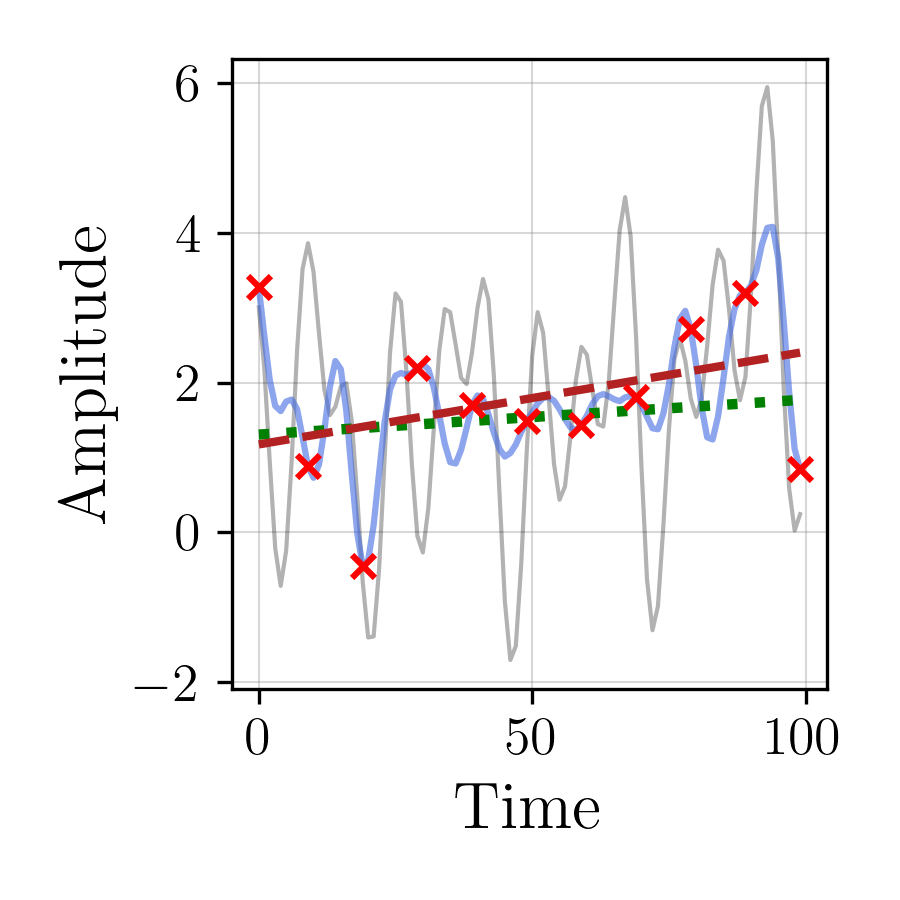}

         \label{fig:reconstr0}
     \end{subfigure}
     \begin{subfigure}[b]{0.49\linewidth}
         \centering
         \includegraphics[width=\textwidth]{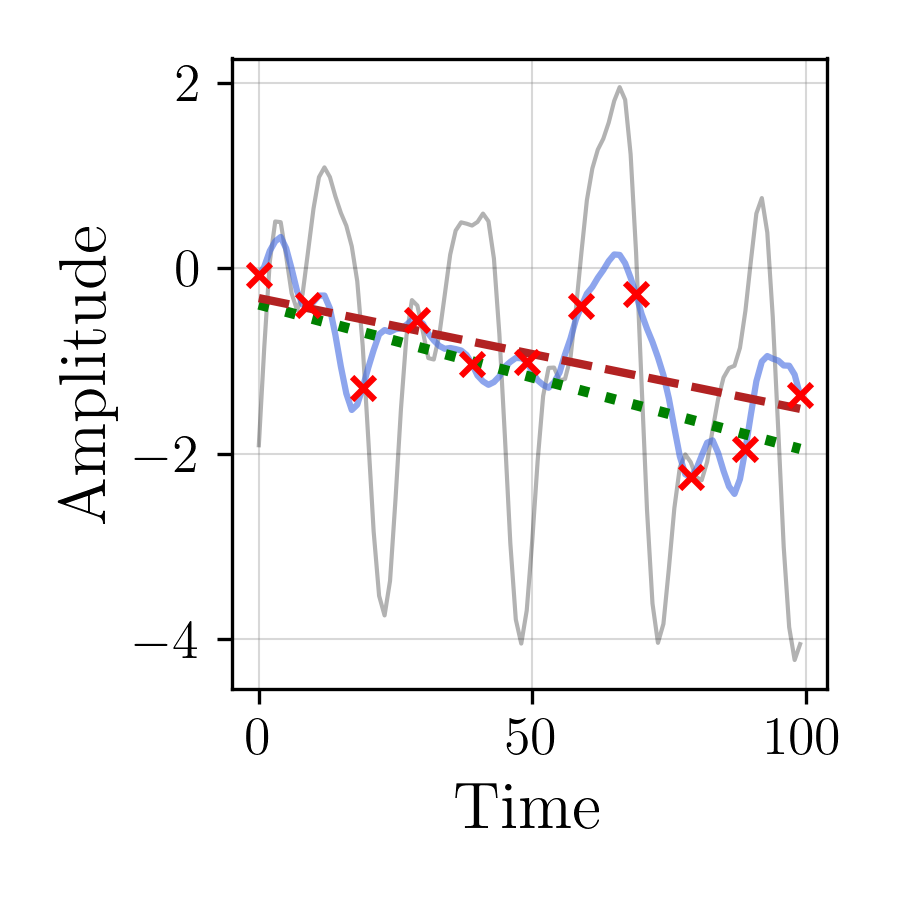}

         \label{fig:reconstr1}
     \end{subfigure}
        \caption{Visual validation of the trend estimation on two dataset samples. 
    The plots display the input signal (gray), the deseasonalized input (blue), and the ground truth trend (green dotted). 
    The red markers indicate the specific observations selected for the OLS regression, yielding the estimated trend (red dashed).}
        \label{fig:Trend_estimation}
\end{figure}

The same procedure is applied to the PJM Hourly and Thermal Power Plant (TPP) datasets.
For PJM, using a weekly window ($W=168$), spectral analysis identifies two dominant harmonics at $6.97$ and $13.99$ cycles/week, as illustrated in \autoref{fig:FFT_ET}, capturing $\rho_{\text{spec}}=0.960$ of the oscillatory energy. Low-frequency components are attributed to trend behavior and handled by the linear encoder. Consequently, a two-stage architecture is sufficient.

\begin{figure}
    \centering
    \includegraphics[width=1\linewidth]{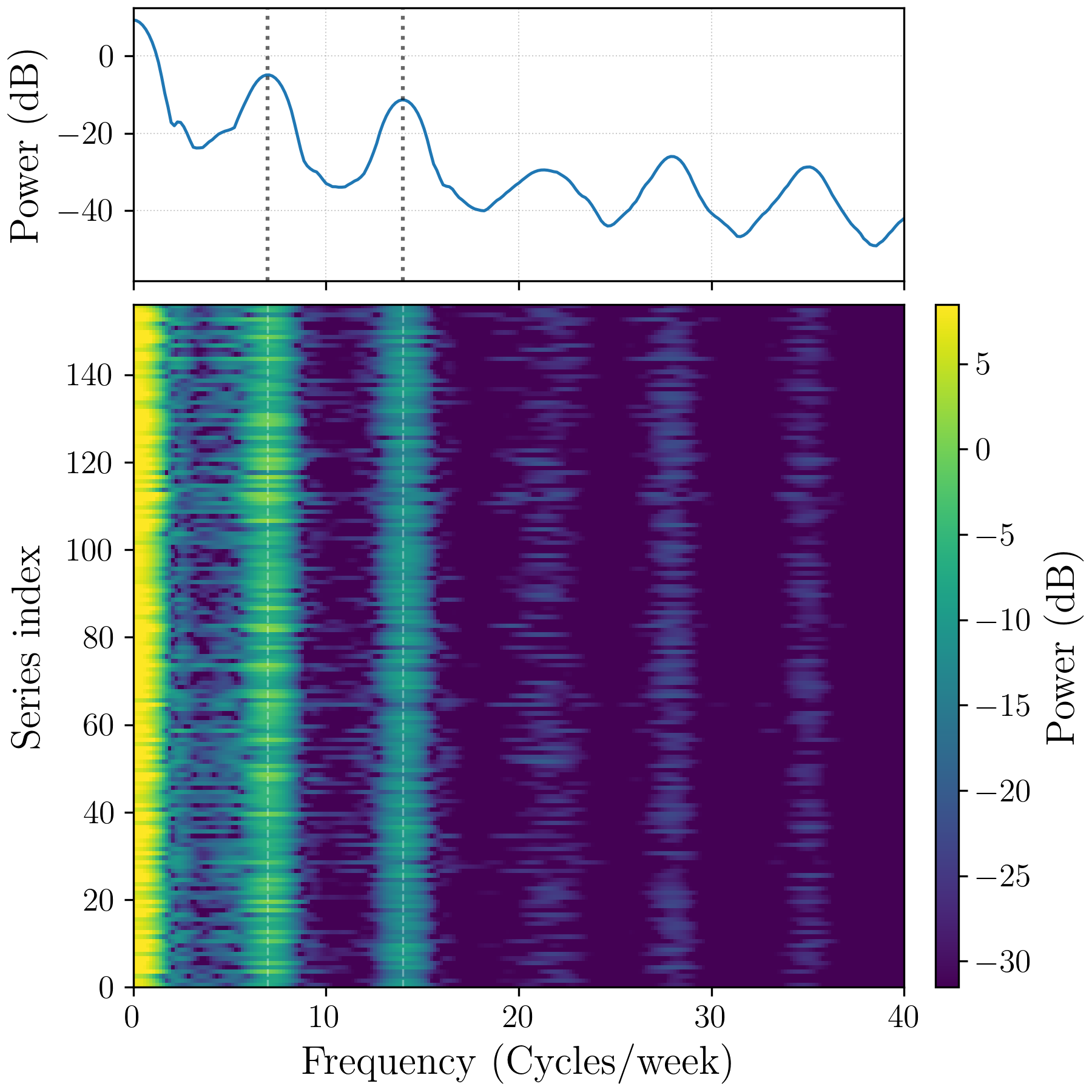}
    \caption{Spectral analysis of the PJM Hourly dataset used for prior extraction. The average periodogram identifies dominant modes at $6.97$ and $13.99$ cycles/week, while the accompanying spectral heatmap displays the energy distribution across all 156 samples. The cumulative energy ratio reaches $\rho_{\text{spec}} = 0.960$ by combining the two identified dominant modes.}
    \label{fig:FFT_ET}
\end{figure}

In contrast, the TPP dataset exhibits a richer spectral structure (\autoref{fig:FFT_TPP}), with four significant modes at $1.42$, $6.95$, $14.00$, and $21.20$ cycles/week. Notably, the low-frequency component at $1.42$ cycles/week contributes $21.8\%$ of the total energy and therefore requires explicit modeling as a seasonal stage. The resulting configuration adopts $|\mathcal{K}|=4$, achieving near-complete spectral coverage ($\rho_{\text{spec}}=0.991$).

\begin{figure}
    \centering
    \includegraphics[width=1\linewidth]{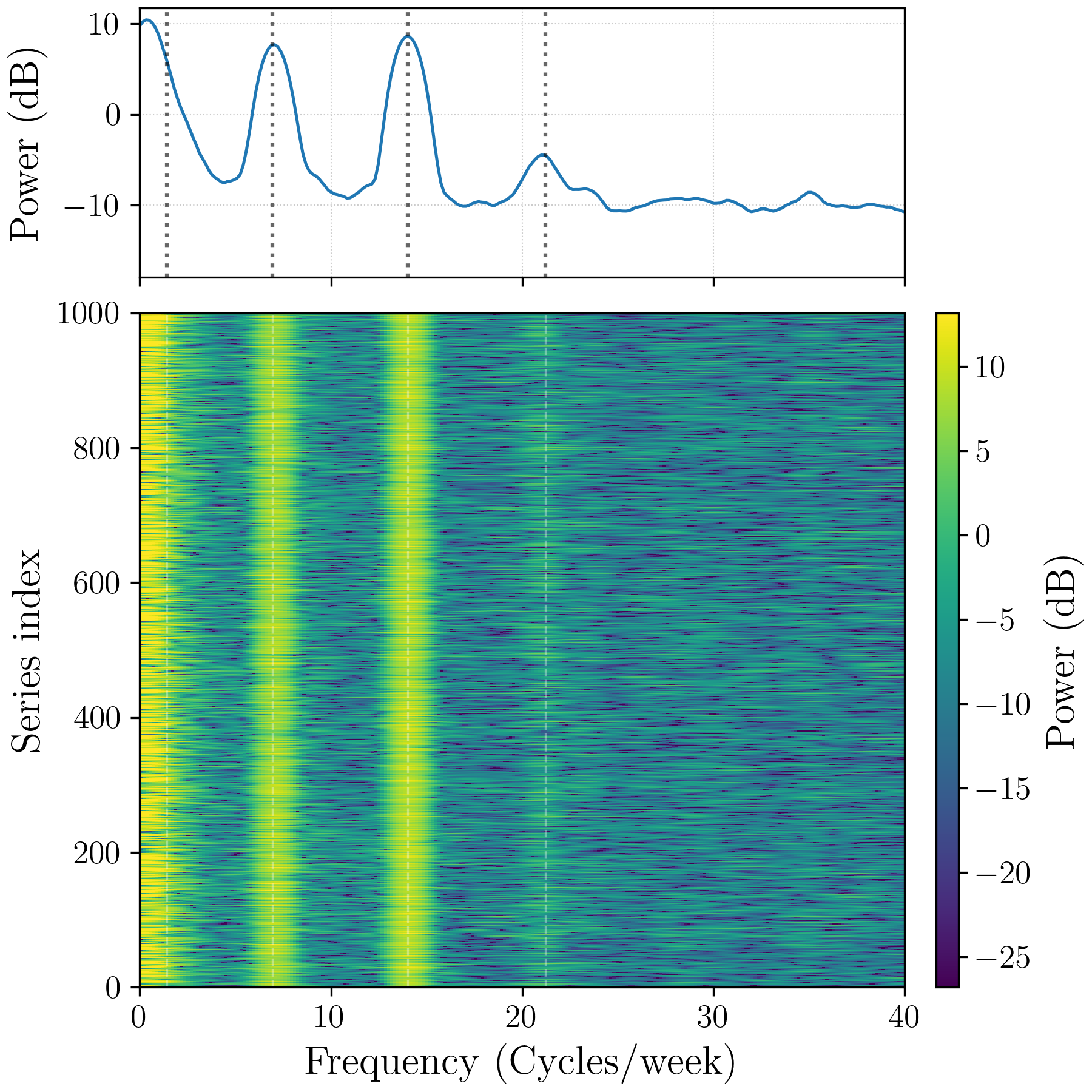}
    \caption{Spectral analysis of the Thermal Power Plant dataset used for prior extraction. The average periodogram identifies distributed dominant modes at $1.42$, $6.95$, $14.00$, and $21.20$ cycles/week, complemented by the spectral heatmap across all 1000 samples. The cumulative energy ratio reaches $\rho_{\text{spec}} = 0.991$ by combining the four identified modes.}
    \label{fig:FFT_TPP}
\end{figure}

Trend priors are estimated from the deseasonalized residuals for both datasets.
As shown in \autoref{fig:RealWorld_Trends}, the smooth residual structure of PJM leads to a compact configuration with $n_{opt}=3$, whereas the higher volatility of the TPP residuals requires $n_{opt}=13$ to satisfy the same stability tolerance of $\delta=0.1$ ($10\%$).
Despite these differences, both cases achieve substantial reductions in trend input dimensionality ($>90\%$) while preserving estimation reliability.

\begin{figure}[htbp]
     \centering
     \begin{subfigure}[b]{0.48\linewidth}
         \centering
         \includegraphics[width=\linewidth]{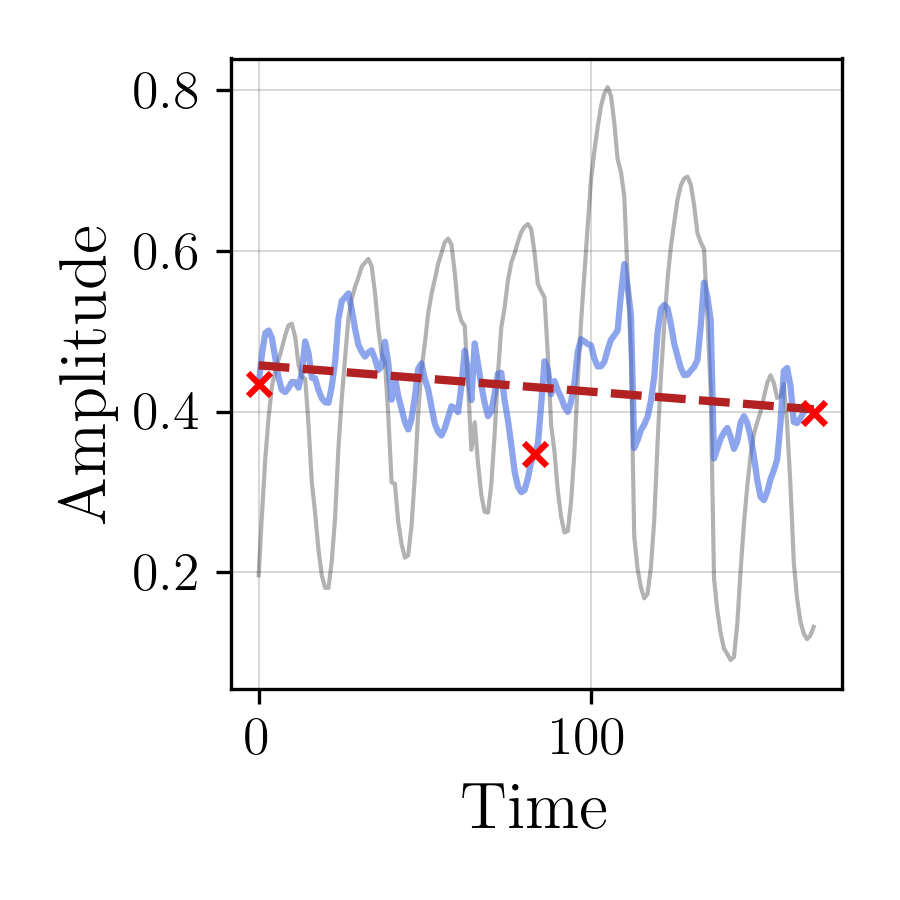} 
         \caption{PJM Hourly}
         \label{fig:trend_pjm}
     \end{subfigure}
     \hfill
     \begin{subfigure}[b]{0.48\linewidth}
         \centering
         \includegraphics[width=\linewidth]{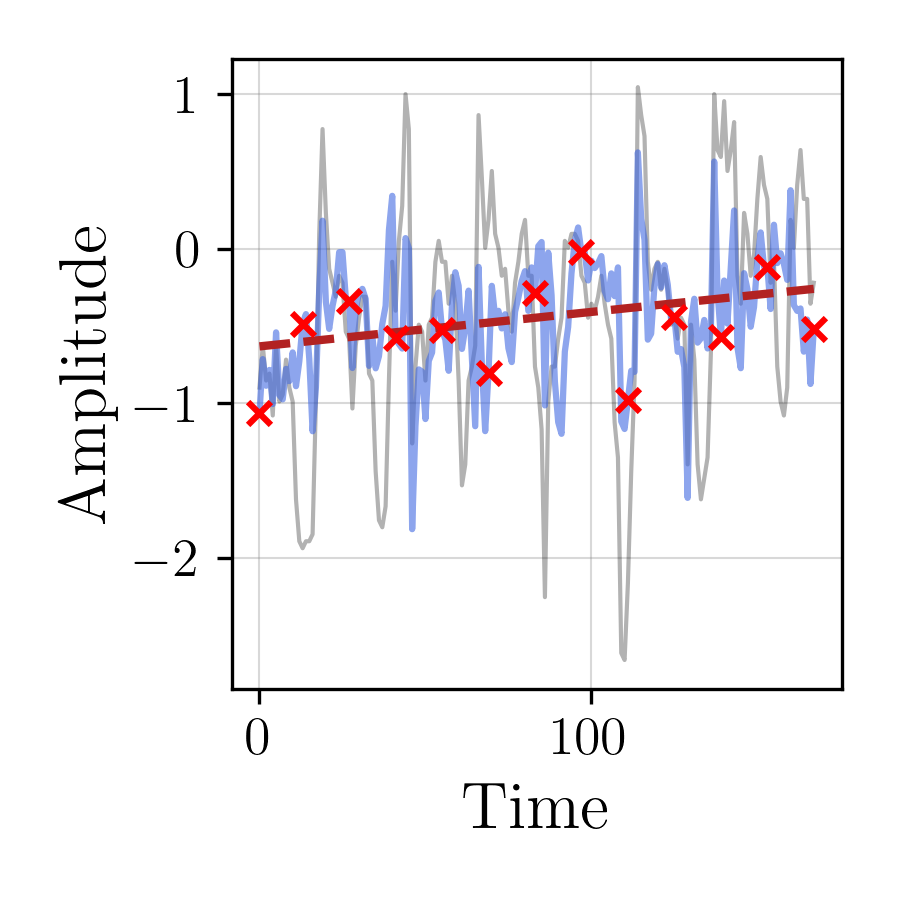} 
         \caption{Thermal Power Plant}
         \label{fig:trend_thermal}
     \end{subfigure}
     \caption{Trend input configuration based on deseasonalized residuals. The blue lines represent the signal proxy obtained by removing the detected harmonics $\mathcal{K}$. The red markers indicate the $n_{opt}$ equidistant observations selected to maximize stability, dimensioned according to the estimated residual volatility.}
     \label{fig:RealWorld_Trends}
\end{figure}

\autoref{tab:architecture_priors} summarizes the resulting data-driven configurations.
These parameters specify the BoF architecture before training, determining its depth, frequency initialization, and trend encoder input size.
The extracted priors serve as informed starting points rather than fixed constraints, allowing the network to refine these estimates during optimization and capture non-linear dynamics beyond the resolution of classical spectral and linear models.

\begin{table*}[t!]
\centering
\scriptsize
\setlength{\tabcolsep}{4pt}
\renewcommand{\arraystretch}{1.25}

\caption{Data-Driven Configuration and Initialization Priors. 
Consolidated architectural parameters derived from the two-stage diagnostic process. 
The table segregates the hyperparameters obtained from spectral decomposition (Seasonal Priors) and residual analysis (Trend Priors).}
\label{tab:architecture_priors}

\begin{tabular}{l c | c c P{4.5cm} | c c}
\toprule
\multirow{2}{*}{\textbf{Dataset}} & \multirow{2}{*}{\textbf{Win.} ($W$)} & \multicolumn{3}{c|}{\textbf{Frequency Analysis (Seasonal Priors)}} & \multicolumn{2}{c}{\textbf{Trend Analysis (Trend Priors)}} \\
\cmidrule(lr){3-5} \cmidrule(lr){6-7}
 & & \textbf{Ratio} ($\rho$) & \textbf{Depth} & \textbf{Mode Initialization} ($\mathcal{K}$) & \textbf{Noise} ($\hat{\sigma}$) & \textbf{Input} ($n_{opt}$) \\
\midrule
\textbf{Synthetic} & 100 & 0.873 & 3 & $[3.48, 7.56, 12.29]$ Hz & 0.850 & 11 \\
\textbf{PJM Hourly} & 168 & 0.960 & 2 & $[6.97, 13.99]$ cycles/win & 0.048 & 3 \\
\textbf{Thermal Plant} & 168 & 0.991 & 4 & $[14.00, 6.95, 1.42, 21.20]$ cycles/win & 0.486 & 13 \\
\bottomrule
\end{tabular}
\end{table*}

With the architectural topology and initialization states explicitly defined, we proceed to the empirical validation. The following section quantifies the practical impact of these data-driven priors on training dynamics, comparing the convergence speed and final accuracy of the proposed strategy against standard paradigms.

\subsection{Results}

This section reports the empirical evaluation of the proposed framework. We first analyze optimization behavior on a synthetic dataset, where ground-truth structure is known. We then assess generalization performance on the two real-world datasets, followed by an evaluation of computational efficiency.

\subsubsection{Validation on Synthetic Data}

The synthetic dataset enables a controlled analysis of how structural priors affect optimization stability and parameter convergence. \autoref{fig:boxplots_syn} summarizes the final training and testing MSE distributions over 10 independent runs per architecture.

\begin{figure}[htb]
    \centering
    \includegraphics[width=1\linewidth]{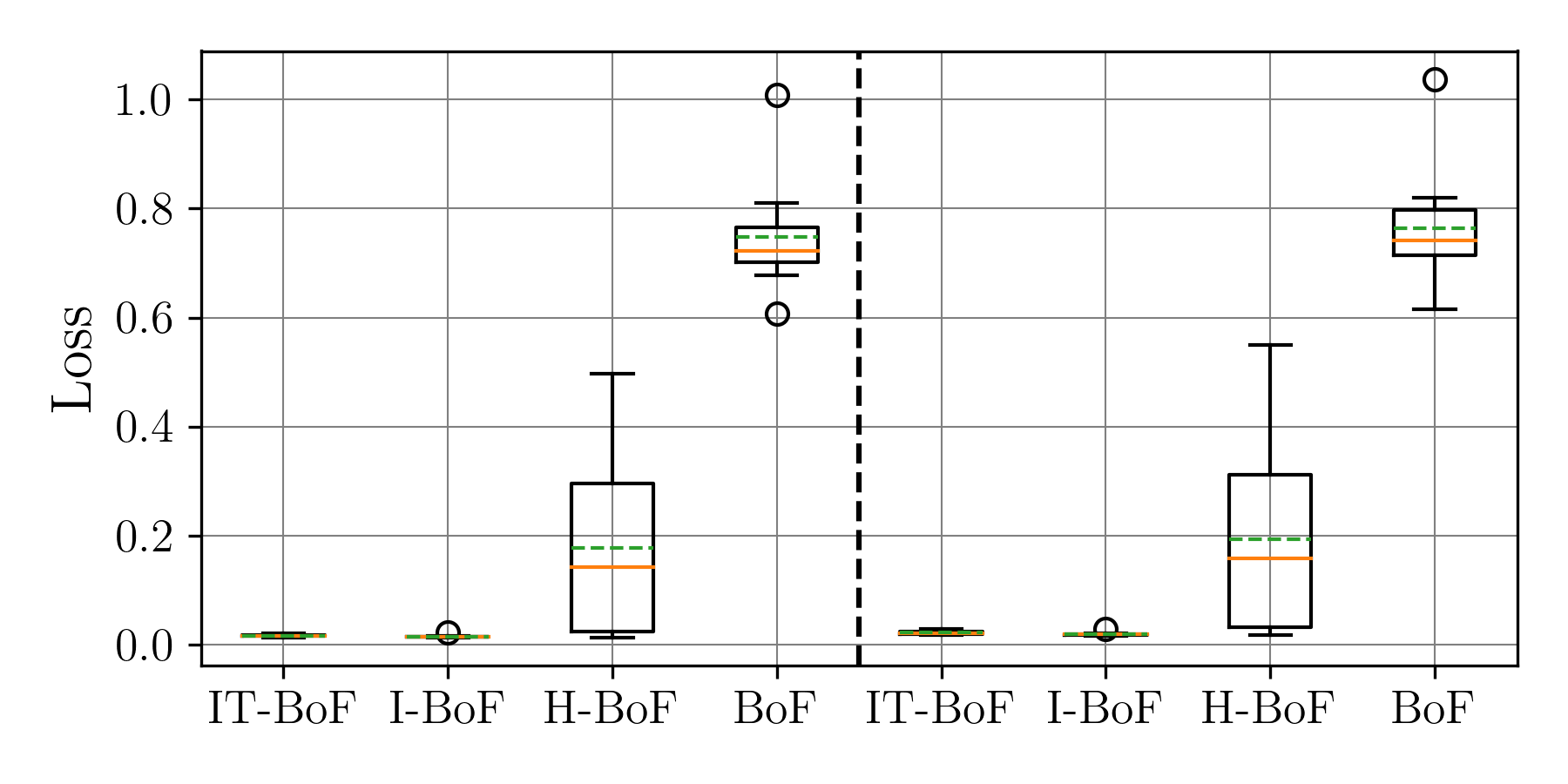}
    \caption{Distribution of the final MSE loss on the training (left) and testing (right) sets. The results summarize 10 independent trials for each of the four evaluated architectures. Reported values represent the average loss at the converged state. Within each box, the solid orange and dashed green lines denote the median and the mean, respectively.}
    \label{fig:boxplots_syn}
\end{figure}

The informed architectures (I-BoF and IT-BoF) achieve the lowest reconstruction error with minimal variance. IT-BoF attains a test MSE of $0.0220 \pm 0.0033$, statistically comparable to I-BoF ($0.0198 \pm 0.0034$), despite its reduced trend dimensionality. In contrast, the baseline BoF exhibits significantly higher error ($0.7633 \pm 0.1091$). Heuristic initialization (H-BoF) improves performance ($0.1937$ mean MSE) but remains less stable. Overall, informed initialization reduces error by approximately 97\% relative to the baseline, confirming that structural priors effectively stabilize the optimization.

To examine convergence behavior, \autoref{fig:convergence_syn} shows the loss trajectories of the best-performing trial per architecture.

\begin{figure}[htb]
    \centering
    \includegraphics[width=1\linewidth]{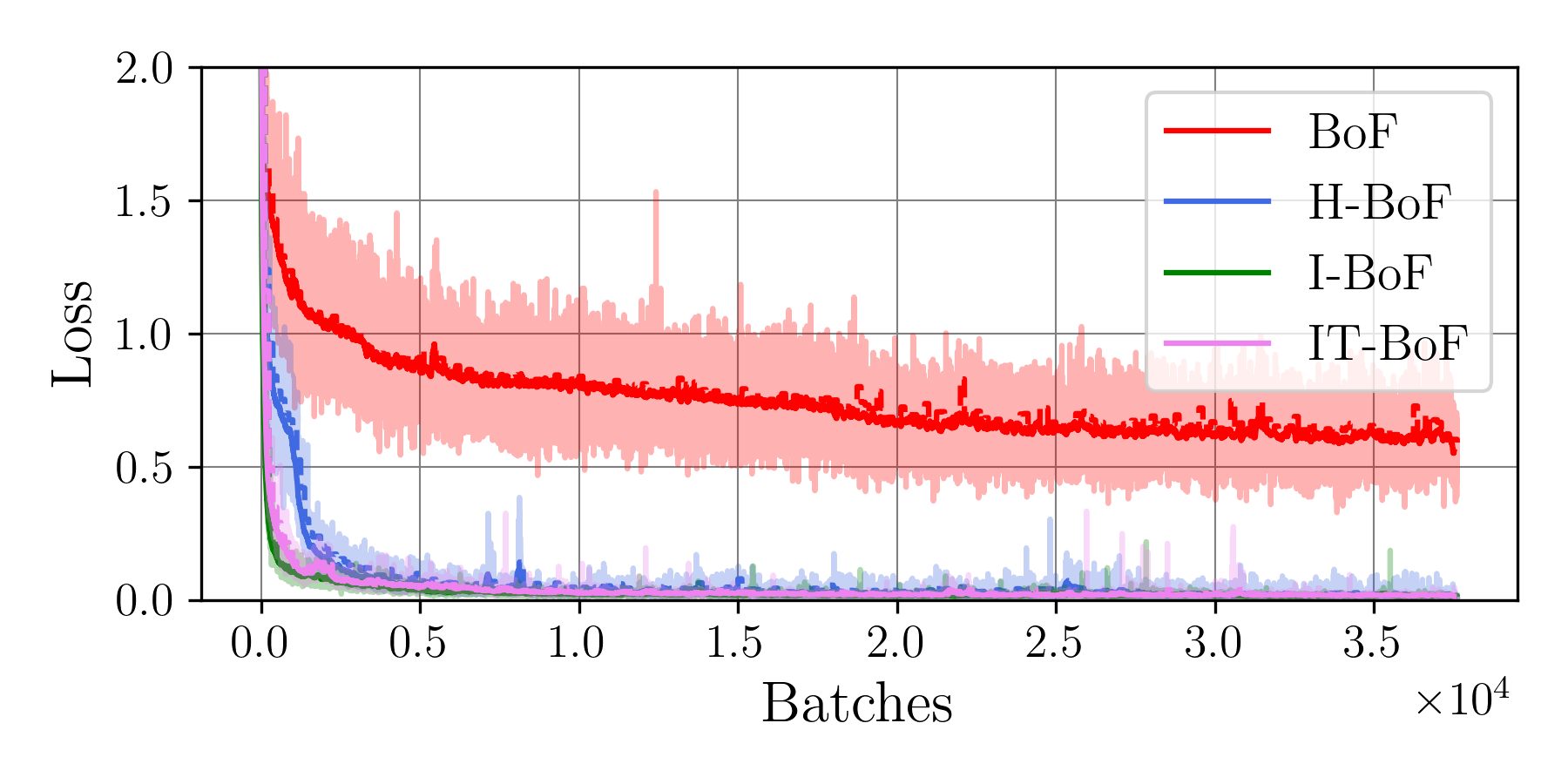}
    \caption{Loss trajectories of the best-performing trial for each architecture. The solid lines represent the moving average, while the lighter shaded regions indicate batch volatility.}
    \label{fig:convergence_syn}
\end{figure}

Although all models exhibit relatively high initial errors, the informed variants start with a substantially lower loss. While the baseline BoF converges slowly and H-BoF shows moderate improvement, both I-BoF and IT-BoF rapidly reach low-error regimes within early iterations. This confirms that data-driven priors substantially reduce optimization effort, a benefit that IT-BoF preserves despite its compact trend representation.

\autoref{fig:parameters_synthetic} reports the cumulative parameter displacement during training.

\begin{figure}[htb]
    \centering
    \includegraphics[width=1\linewidth]{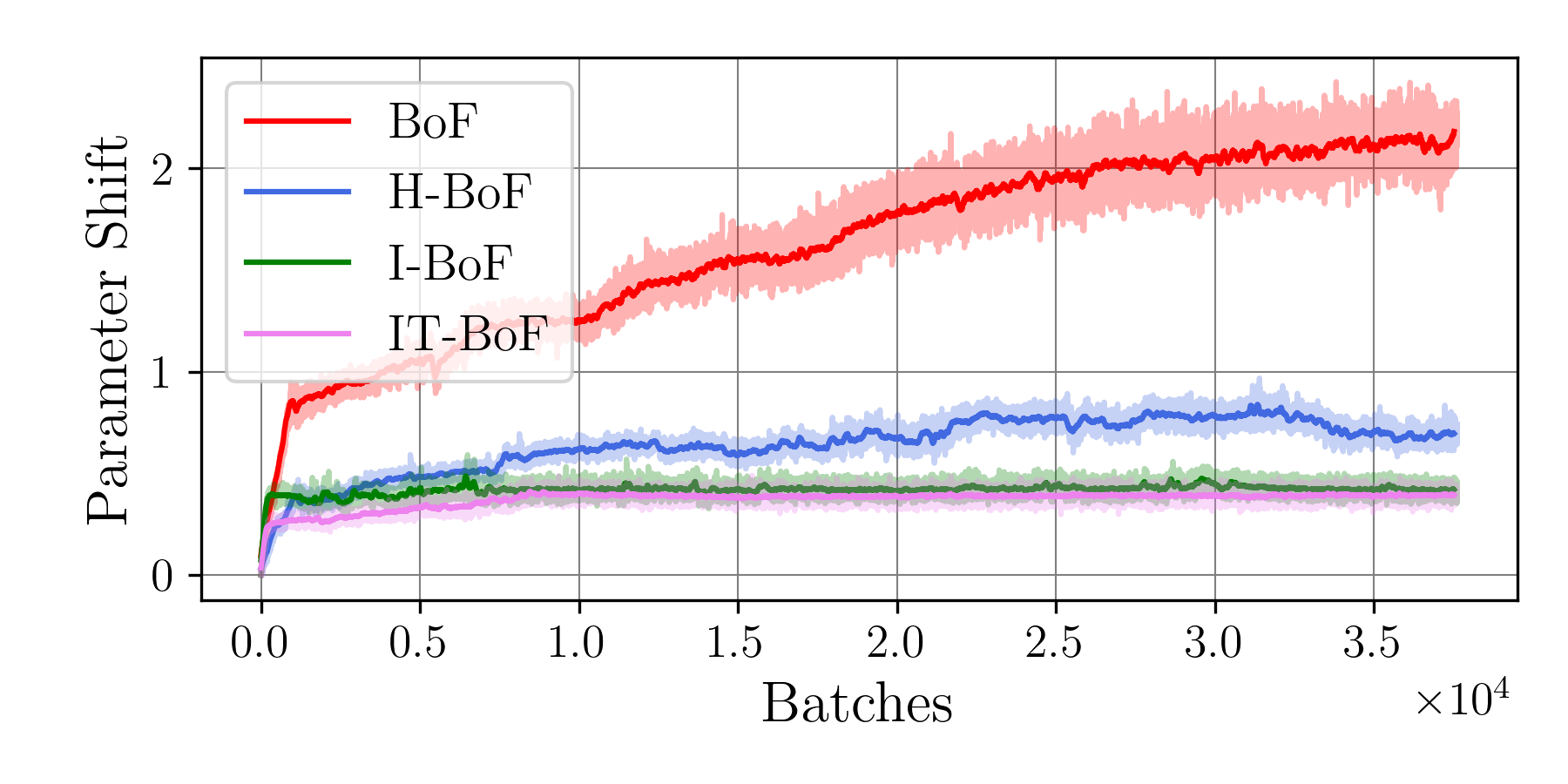}
    \caption{Evolution of the parameter displacement during training for the best-performing trial of each architecture. Solid lines represent moving averages, while shaded regions indicate batch-wise variability.}
    \label{fig:parameters_synthetic}
\end{figure}

The baseline BoF undergoes large and sustained parameter shifts, whereas informed architectures stabilize quickly after a brief adjustment phase. IT-BoF exhibits the smallest overall displacement, indicating that informed dimensionality reduction not only preserves accuracy but further constrains parameter drift, leading to more stable learning dynamics.

\subsubsection{Generalization to Real-World Scenarios}

We next evaluate generalization on two real-world datasets: PJM electricity demand and the Thermal Power Plant (TPP) dataset. \autoref{tab:results_real_world} reports reconstruction performance averaged over 10 independent trials.

\begin{table}[htbp]
\centering
\small
\caption{Reconstruction performance on real-world benchmarks. The table reports the Mean Squared Error for the training and testing partitions across the PJM and Thermal Power Plant (TPP) datasets. Results are averaged over ten independent trials per architecture, and the best performance per dataset is highlighted in bold.}
\label{tab:results_real_world}
\resizebox{\columnwidth}{!}{%
\begin{tabular}{l c c c c}
\toprule
\multicolumn{5}{c}{\textbf{Training MSE}} \\
\midrule
\textbf{Dataset} & \textbf{BoF} & \textbf{H-BoF} & \textbf{I-BoF} & \textbf{IT-BoF} \\
\midrule
PJM & $0.0116 \scriptstyle{\pm 0.0005}$ & $0.0066 \scriptstyle{\pm 0.0026}$ & $0.0060 \scriptstyle{\pm 0.0007}$ & $\mathbf{0.0046 \scriptstyle{\pm 0.0008}}$ \\
TPP & $0.4423 \scriptstyle{\pm 0.0178}$ & $0.1957 \scriptstyle{\pm 0.0431}$ & $\mathbf{0.1715 \scriptstyle{\pm 0.0025}}$ & $0.1755 \scriptstyle{\pm 0.0043}$\\
\midrule
\multicolumn{5}{c}{\textbf{Testing MSE}} \\
\midrule
PJM & $0.0155 \scriptstyle{\pm 0.0008}$ & $0.0130 \scriptstyle{\pm 0.0131}$ & $0.0099 \scriptstyle{\pm 0.0024}$ & $\mathbf{0.0074 \scriptstyle{\pm 0.0011}}$ \\
TPP & $0.4621 \scriptstyle{\pm 0.0192}$ & $0.2309 \scriptstyle{\pm 0.0481}$ & $\mathbf{0.1958 \scriptstyle{\pm 0.0027}}$ & $0.2035 \scriptstyle{\pm 0.0077}$ \\
\bottomrule
\end{tabular}%
}
\end{table}

Across both datasets, informed architectures consistently outperform the baseline and heuristic variants. On PJM, characterized by smooth dynamics, IT-BoF achieves the best training and testing performance with a small generalization gap, reducing the testing MSE from $0.0155$ (BoF) to $0.0074$, which corresponds to a relative improvement of approximately $52\%$. On the more complex TPP dataset, I-BoF attains the lowest testing error, lowering the MSE from $0.4621$ to $0.1958$ (about $58\%$ reduction), while IT-BoF remains competitive despite its reduced dimensionality, achieving a comparable $56\%$ improvement over the baseline. These results confirm that the benefits of data-driven priors transfer robustly beyond controlled synthetic settings. 
\autoref{fig:RealWorld_Parameters} illustrates the evolution of parameter displacement during training.

\begin{figure}[htbp]
     \centering
     \begin{subfigure}[b]{0.48\linewidth}
         \centering
         \includegraphics[width=\linewidth]{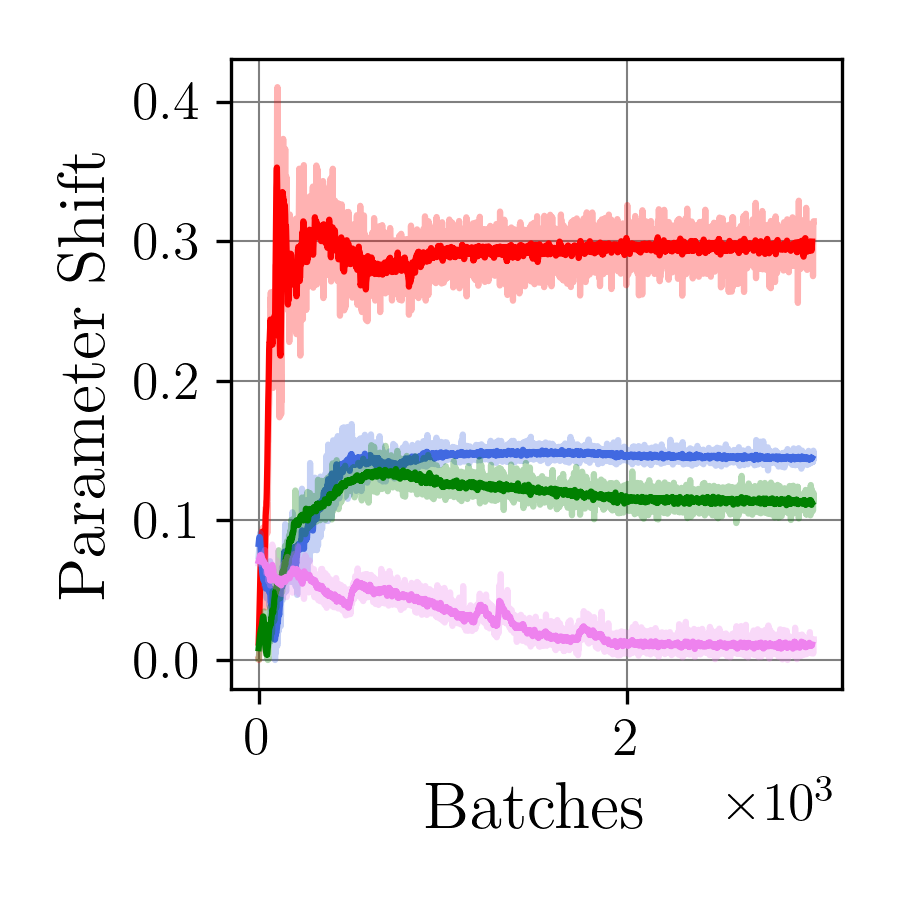} 
         \caption{PJM Hourly}
         \label{fig:parameters_pjm}
     \end{subfigure}
     \hfill
     \begin{subfigure}[b]{0.48\linewidth}
         \centering
         \includegraphics[width=\linewidth]{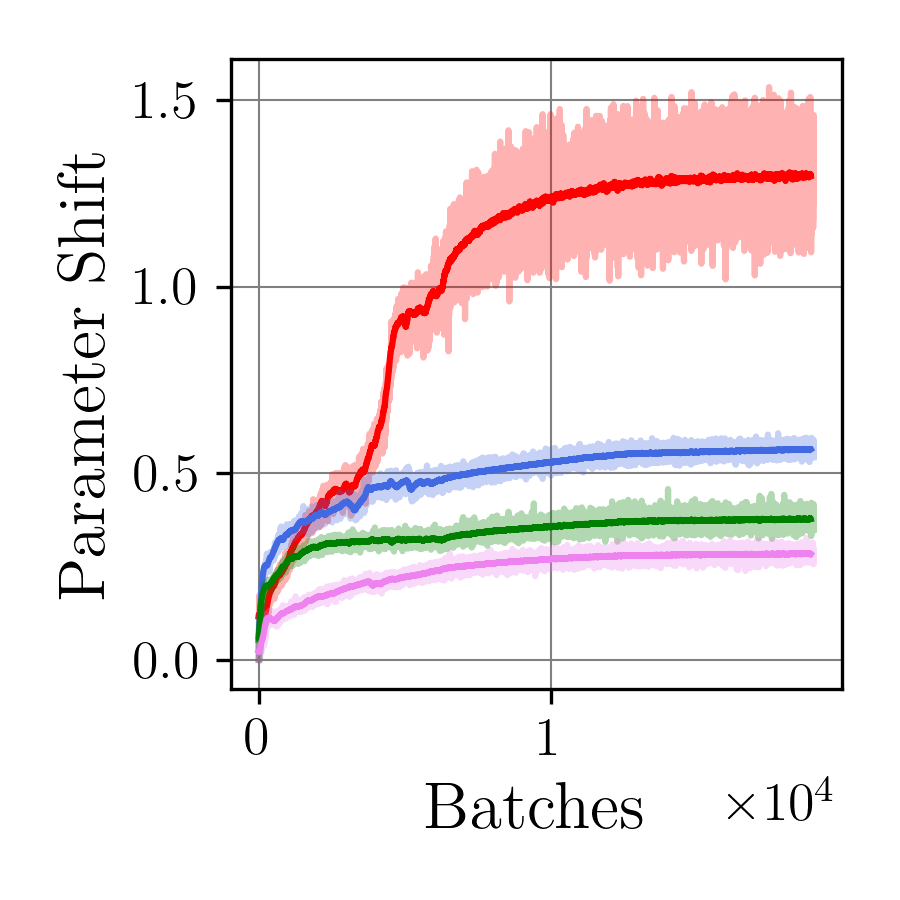} 
         \caption{Thermal Power Plant}
         \label{fig:parameters_thermal}
     \end{subfigure}
     \caption{Evolution of parameter displacement during training on real-world datasets for the best-performing trial of each architecture. Results are shown for the PJM (left) and Thermal Power Plant (right) datasets. Parameter displacement is measured relative to the initialization and averaged across network parameters. Solid lines denote moving averages, while shaded regions indicate batch-wise variability. Colors correspond to BoF (red), H-BoF (blue), I-BoF (green), and IT-BoF (pink).}
     \label{fig:RealWorld_Parameters}
\end{figure}

For both datasets, informed architectures exhibit substantially lower and more stable parameter drift than BoF, indicating that improved reconstruction accuracy is accompanied by constrained optimization trajectories. IT-BoF consistently shows the smallest long-term displacement, reflecting the stabilizing effect of informed initialization under realistic operating conditions.

While the previous analysis highlights the benefits of structural priors in terms of reconstruction accuracy and optimization stability, practical deployment in real-world systems additionally requires favorable computational characteristics. We therefore complement the performance analysis with an evaluation of the computational efficiency of the proposed architectures.

\subsubsection{Computational Efficiency}

Beyond reconstruction accuracy and optimization behavior, the practical applicability of the proposed framework depends critically on its computational efficiency. \autoref{tab:efficiency_metrics} summarizes model size, computational complexity, and inference performance across datasets. 

\begin{table}[ht]
\centering
\caption{Computational efficiency comparison (Parameters, FLOPs, Latency, and Throughput) for Synthetic and PJM datasets.}
\label{tab:efficiency_metrics}
\resizebox{\columnwidth}{!}{%
\begin{tabular}{l l c c c c}
\toprule
\textbf{Dataset} & \textbf{Model} & \textbf{Params ($\times 10^3$)} & \textbf{FLOPs ($\times 10^3$)} & \textbf{Latency (ms)} & \textbf{Throughput (samples/s)} \\
\midrule
\multirow{4}{*}{\textbf{Synthetic}} 
 & BoF    & \multirow{3}{*}{62.25} & \multirow{3}{*}{61.54} & 4.484 & 3895.1 \\
 & H-BoF  &                        &                        & 4.417 & 3879.2 \\
 & I-BoF  &                        &                        & 4.401 & 3388.5 \\
 & IT-BoF & 47.38                  & 46.81                  & 4.329 & 4012.0 \\
\midrule
\multirow{4}{*}{\textbf{PJM}}       
 & BoF    & \multirow{3}{*}{63.23} & \multirow{3}{*}{62.75} & 2.993 & 5749.1 \\
 & H-BoF  &                        &                        & 2.993 & 5886.0 \\
 & I-BoF  &                        &                        & 2.747 & 6144.8 \\
 & IT-BoF & 44.89                  & 44.53                  & 2.662 & 6223.9 \\
\midrule
\multirow{4}{*}{\textbf{TPP}}       
 & BoF    & \multirow{3}{*}{126.55} & \multirow{3}{*}{125.60} & 5.838 & 2980.7 \\
 & H-BoF  &                        &                        & 5.603 & 3160.8 \\
 & I-BoF  &                        &                        & 5.663 & 3131.2 \\
 & IT-BoF & 97.69                  & 96.90                 & 5.405 &  3198.5\\
\bottomrule
\end{tabular}%
}
\end{table}

While BoF, H-BoF, and I-BoF share similar architectural complexity, IT-BoF achieves a consistent 20–30\% reduction in parameters and FLOPs due to informed trend dimensionality reduction. These reductions translate into improved throughput (up to 8\%) and comparable or lower latency across datasets, without altering the training procedure or model expressiveness.

Overall, the proposed framework improves reconstruction accuracy, optimization stability, and computational efficiency simultaneously, supporting its suitability for scalable and resource-constrained time-series applications.

\subsection{Comparative Benchmarking against State-of-the-Art}

Having established the impact of informed initialization on training dynamics and convergence behavior, we now assess whether these advantages translate into measurable gains under standardized benchmarking protocols. To this end, we evaluate the proposed initialization strategy within a widely adopted comparative setting for time-series generative models, enabling a direct performance-based comparison against representative GAN- and VAE-based baselines.

Rather than introducing a new architectural variant, this benchmark is designed to isolate the effect of initialization by applying the informed prior exclusively at the optimization level. All competing models are evaluated under identical preprocessing, training, and evaluation conditions, ensuring that observed differences can be attributed to the proposed strategy rather than to architectural or procedural discrepancies.

We report results on four datasets comprising Sine, Stocks, Air Quality, and PU, covering both synthetic and real-world domains with varying temporal structure and noise characteristics. Performance is assessed using discriminative and predictive metrics, which jointly capture distributional fidelity and downstream forecasting utility. \autoref{tab:empirical_results} summarizes the comparative results.

\begin{table*}[ht]
\centering
\scriptsize
\setlength{\tabcolsep}{2pt}
\caption{Empirical Test Data Results. The discriminative and predictive tasks are performed across four different time series datasets and ten synthetic data generators. The predictive scores are compared to the original train and test data domain.}
\label{tab:empirical_results}
\begin{tabular}{c|c|P{0.9cm}P{0.9cm}P{0.9cm}P{0.9cm}P{0.9cm}P{0.9cm}P{0.9cm}P{0.9cm}P{0.9cm}P{0.9cm}P{1.1cm}P{1.1cm}}
\toprule
\textbf{Dataset} & \textbf{Metric} & \textbf{Original} & \textbf{R-GAN} \cite{esteban2017real} & \textbf{LSTM-GAN} \cite{leznik2021multivariate} & \textbf{ITF-GAN} \cite{klopries2023extracting} & \textbf{CR-VAE} \cite{li2023causal} & \textbf{Wave-GAN} \cite{donahueadversarial} & \textbf{CRNN-GAN}\cite{mogren2016c} & \textbf{Time-VAE} \cite{desai2021timevae} & \textbf{ITF-VAE} \cite{klopries2025itf} & \textbf{ITF-VAE-I} \cite{klopries2025itf} & \textbf{ITF-VAE-Init}\cite{torres2025toward}& \textbf{ITF-VAE-Informed Init}\\
\midrule
\multirow{2}{*}{Sine} 
& Discriminative & - & 0.1275 & 0.3255 & 0.2274 & 0.1756 & 0.4970 & 0.2580 & 0.2705 & 0.0740 & 0.1910 & 0.0660 & \textbf{0.0550}\\
& Predictive     & 0.0320 & 0.1000 & 0.1162 & 0.1061 & 0.1008 & 0.2566 & 0.0988 & 0.1616 & 0.0612 & 0.0707 & 0.0103 & \textbf{0.0087}\\
\midrule
\multirow{2}{*}{Stocks} 
& Discriminative & - & 0.2143 & 0.4922 & 0.2325 & 0.4627 & 0.4582 & 0.1254 & 0.3632 & 0.1002 & 0.2709 & 0.0099 & \textbf{0.0050}\\
& Predictive     & 0.0730 & 0.1020 & 0.1145 & 0.0967 & 0.1857 & 0.1180 & 0.1042 & 0.1200 & 0.0881 & 0.2395 & 0.0180 & \textbf{0.0170}\\
\midrule
\multirow{2}{*}{Air} 
& Discriminative & - & 0.3809 & 0.4935 & 0.4521 & 0.4556 & 0.4878 & 0.3387 & \textbf{0.3015} & 0.3955 & 0.4376 & 0.3995 & 0.3565\\
& Predictive     & 0.2766 & \textbf{0.3806} & 0.5954 & 0.4787 & 0.5073 & 0.6428 & 0.4137 & 0.3850 & 0.5112 & 0.6136 & 0.4246 & 0.4145\\
\midrule
\multirow{2}{*}{PU} 
& Discriminative & - & 0.4880 & 0.2530 & 0.2683 & 0.3885 & 0.3100 & 0.2930 & 0.2480 & 0.2085 & 0.3335 & 0.0785 & \textbf{0.0685} \\
& Predictive     & 0.1163 & 0.1774 & 0.1536 & 0.1219 & 0.1611 & 0.1513 & 0.1340 & 0.1458 & 0.1336 & 0.1629 & 0.0480 & \textbf{0.0461}\\
\bottomrule
\end{tabular}
\end{table*}

The results in \autoref{tab:empirical_results} indicate that informed initialization yields consistent and substantial improvements across multiple datasets and evaluation criteria. In particular, the ITF-VAE-Informed Init achieves the lowest discriminative and predictive scores on the Sine, Stocks, and PU datasets, outperforming both adversarial and variational baselines that rely on data-driven initialization schemes.

These findings suggest that embedding physically motivated priors at initialization time improves not only optimization efficiency but also the quality of the learned generative distribution. Importantly, these gains are achieved without modifying the underlying model architecture, highlighting initialization as an effective and generally applicable lever for enhancing time-series generative models.

\section{Conclusion}
\label{sec:Conclusion}

This work examined the impact of domain-informed initialization on the optimization behavior, generalization performance, and efficiency of neural function-parameterizing architectures. By embedding spectral and trend priors directly into the initialization process, we demonstrated that meaningful structural information can be exploited without modifying the underlying model architecture or training procedure.

Experiments on synthetic and real-world datasets showed that informed initialization consistently stabilizes optimization, accelerates convergence, and reduces reconstruction error relative to standard and heuristic baselines. These gains are accompanied by more constrained parameter evolution during training, indicating improved optimization trajectories and robust behavior across independent trials.

Beyond accuracy, the proposed framework delivers tangible computational benefits. By leveraging trend-informed dimensionality reduction, the IT-BoF architecture achieves substantial reductions in model size and computational cost, resulting in higher throughput and comparable or lower inference latency. These results highlight the practical value of principled initialization strategies for scalable and resource-constrained time-series applications.

Benchmarking against state-of-the-art generative models further confirms that informed initialization improves discriminative and predictive performance across diverse datasets, despite operating exclusively at the optimization level. This underscores initialization as a powerful and broadly applicable mechanism for enhancing time-series models.

Future work will focus on exploiting data-driven prior extraction to further automate the design of compact, adaptive architectures whose complexity is directly matched to the intrinsic structure of the data. In addition, coupling the proposed framework with adaptive or online learning settings represents a promising direction, enabling architectures to evolve in response to changing signal characteristics.

\bibliographystyle{IEEEtran}
\bibliography{references}

\end{document}